\documentclass{article}

\usepackage[main, final]{neurips_2026}

\usepackage[utf8]{inputenc}
\usepackage[T1]{fontenc}
\usepackage{amsmath, amssymb, amsthm}
\usepackage{graphicx}
\usepackage{xcolor}
\usepackage{url}
\usepackage{booktabs}
\usepackage{nicefrac}
\usepackage{microtype}
\usepackage{float}
\usepackage{bm}
\usepackage{algorithm}
\usepackage{algpseudocode}
\usepackage{aliascnt}
\usepackage{hyperref}
\usepackage[nameinlink,capitalize]{cleveref}
\makeatletter
\providecommand*{\theHALG@line}{}
\renewcommand*{\theHALG@line}{\thealgorithm.\arabic{ALG@line}}
\makeatother
\crefname{assumption}{Assumption}{Assumptions}
\crefname{hypothesis}{Hypothesis}{Hypotheses}
\crefname{definition}{Definition}{Definitions}
\crefname{lemma}{Lemma}{Lemmas}
\crefname{proposition}{Proposition}{Propositions}
\crefname{corollary}{Corollary}{Corollaries}
\crefname{theorem}{Theorem}{Theorems}

\newcommand{\hiddenDim}{d}             
\newcommand{\inputDim}{m}              
\newcommand{\outputDim}{p}             
\newcommand{\ustb}{k} 
\newcommand{\taustb}{\tau_{\ustb}}     
\newcommand{\IDR}{\mathrm{IDR}}        

\newcommand{\pred}{\widehat{y}}        
\newcommand{\alg}{\mathcal{A}}         

\newcommand{\TOCITE}[1]{%
  \ifx\relax#1\relax%
    \textcolor{red}{[add citations]}%
  \else%
    \textcolor{red}{[add citations: #1]}%
  \fi%
}

\newtheorem{theorem}{Theorem}
\newaliascnt{lemma}{theorem}
\newtheorem{lemma}[lemma]{Lemma}
\aliascntresetthe{lemma}
\newaliascnt{proposition}{theorem}
\newtheorem{proposition}[proposition]{Proposition}
\aliascntresetthe{proposition}
\newaliascnt{corollary}{theorem}

\aliascntresetthe{corollary}
\theoremstyle{definition}
\newaliascnt{definition}{theorem}
\newtheorem{definition}[definition]{Definition}
\aliascntresetthe{definition}
\newaliascnt{assumption}{theorem}
\newtheorem{assumption}[assumption]{Assumption}
\aliascntresetthe{assumption}
\newaliascnt{hypothesis}{theorem}

\aliascntresetthe{hypothesis}
\theoremstyle{remark}
\newaliascnt{remark}{theorem}

\aliascntresetthe{remark}
\newtheorem{claim}{Claim}[theorem]
\newenvironment{claimproof}
  {\begin{proof}[Proof of claim]}
  {\end{proof}}

\newcommand{\GUFcl}{\mathrm{GUF}}

\title{A Memory Efficient Unified Algorithm for Online Learning of Linear Dynamical Systems}

\author{%
  Yuval Ran-Milo\thanks{Equal contribution. Correspondence to: \texttt{<yuvalmilo@mail.tau.ac.il>} and \texttt{<aa2279@princeton.edu>}.}\\
  Tel Aviv University\\
  \And
  Angelos Assos\footnotemark[1]\\
  Princeton University\\
  \And
  Elad Hazan\\
  Princeton University\\
}

\begin{document}

\maketitle

\begin{abstract}
  Motivated by the challenge of stabilizing a \emph{general} unknown linear dynamical system (LDS) from observations, we study the natural prerequisite of online prediction. 
  Our goal is to achieve sublinear regret with a memory footprint that adapts to the intrinsic complexity of the dynamics rather than the full hidden-state dimension. We focus on the practically central regime of systems with low \emph{instability complexity}---eigenvalues outside the real stable interval that do not decay rapidly, together with non-semisimple modes---potentially embedded in an otherwise stable real spectrum of much higher dimension; we write $\ustb$ for this count.
  This regime is the primary setting in which stabilization is plausible: we show that many systems with high instability complexity cannot be stabilized without exponentially large controls.
    Thus, prediction is meaningful for stabilization precisely when the instability complexity is small. Within this regime, we introduce a unified online algorithm that handles every LDS (including non-diagonalizable systems with complex or exploding modes) with a learnable parameter count of $\widetilde{O}(\ustb)$. 
    Finally, we prove a lower bound showing that $\ustb$ is a valid complexity measure: any filter-based predictor needs at least $\ustb$ filters.
    Experiments corroborate our theory: on a high-dimensional system, our predictor sharply outperforms prior methods at an equal parameter budget.
\end{abstract}

\section{Introduction}
\label{sec:intro}
  
A long-standing goal in control theory is to \emph{stabilize} an unknown
dynamical system from observations: design an input policy that keeps the
system in a bounded region, even when the open-loop dynamics are unstable
\citep{doyle1989state,tsiamis2022hard,foster2020logarithmic,agarwal2019online,abbasi2011regret,faradonbeh2018finite,simchowitz2018learning,sarkar2019near}.
Prediction is a natural prerequisite for stabilization: an algorithm that
cannot track the system's output trajectory has little basis for choosing
inputs that would bring it under control.
We study this prediction problem for linear dynamical systems (LDS) in
the online learning framework: an algorithm observes inputs and outputs
sequentially and at each step must predict the next output before seeing it,
achieving sublinear regret against the best LDS predictor in hindsight,
with no knowledge of the system parameters
\citep{hazan2017spectral,hardt2016gradient,kuznetsov2015learning,tsiamis2020online}. 

A central challenge—and the main contribution of this work—is achieving
sublinear regret with a memory footprint that adapts to the intrinsic complexity
of the dynamics rather than the full hidden-state dimension.
We focus on a regime that is both practically relevant and theoretically
interesting: systems with few modes that are unstable, non-semisimple, or not
rapidly decaying.
Such systems arise naturally---a physical plant with mild resonance, a single
unstable direction embedded in a large stable state space, or a controlled process
with a few runaway modes---yet prior work offers no prediction guarantees that
adapt to this structure.
Moreover, as we show in \cref{sec:lower}, a broad class of systems with many exploding modes relative to the number of input channels is essentially impossible to keep
bounded in practice (\cref{prop:cond-number}).
This makes the small-$\ustb$ regime more than a natural special case: 
it delineates the setting in which one can practically hope to maintain the
system under control.

Existing predictors fall into three useful families, each suited to a different
regime.
Spectral filtering \citep{hazan2017spectral} achieves the strongest
known guarantees for stable systems: a parameter footprint polylogarithmic in $T$
and completely independent of the state dimension $\hiddenDim$.
Spectral filtering, however, is limited to systems whose state matrix is
diagonalizable with all eigenvalues in $(-1,1)$, a provably significant restriction \citep{ranmilo2024provablebenefitscomplexparameterizations}.
Finite-memory input filters (FIR) are also independent of the state dimension,
but apply only to systems whose modes all decay rapidly.
Autoregressive (AR) predictors \citep{ho1966effective,kailath1980linear}, and system identification methods \citep{overschee1996subspace,ljung1999system} can represent any observable LDS
exactly, but at the cost of memory proportional to $\hiddenDim$---even if only one out
of $\hiddenDim$ modes is non-stable.
None of these tools alone handles the mixed case: a system with
$\hiddenDim - \ustb$ easily filtered modes and $\ustb \ll \hiddenDim$
hard ones.

We give a single, \emph{unified} algorithm (\cref{alg:unified}) that closes
this gap, simultaneously achieving what neither prior family could alone.
Our algorithm handles every LDS---including non-diagonalizable systems and
systems with complex or exploding modes---and achieves sublinear regret with a
learnable parameter count of $\widetilde{O}(\ustb)$, where $\ustb$ is the
instability complexity (\cref{def:ustb}).
It thus combines the two benefits that no prior method enjoyed simultaneously:
the \emph{generality} of autoregressive predictors and the \emph{dimension-free
memory footprint} of spectral and finite-memory filtering.
Constructing a stabilizing controller from the learned predictor is
left to future work.

Throughout, we study system-input pairs that keep the dynamics under
control: although the open-loop system may be unstable, the chosen input
sequence keeps the resulting outputs bounded.
This is the regime in which prediction is useful for learning control; once the
realized trajectory has already exploded, the rollout has left the setting where
predictions can be used to keep the system stabilized. 

Finally, in \cref{sec:lower} we show that the \emph{instability complexity} $\ustb$ is a natural and well-motivated parameter from both a theoretical and practical perspective. On the theoretical side, we prove a lower bound showing that the dependence on $\ustb$ is unavoidable: any filter-based predictor with bounded coefficient mass needs at least $\ustb$ filter directions to achieve sublinear regret uniformly over the LDS class (\cref{prop:filter-lb}).
On the practical side, we show that, in a broad class of systems, many exploding modes make stabilization impractical: with high probability over a uniformly random first nonzero input direction, every stabilizing continuation must span an exponentially large range of magnitudes---its input dynamic range (\cref{def:idr}) grows exponentially in $\ustb$ (\cref{prop:cond-number})---making such controllers acutely sensitive to noise, actuator limits, and finite precision.
 These results motivate the regime where $\ustb$ is small compared with $\hiddenDim$, making our algorithm highly efficient.

We complement these results with experiments (\cref{sec:experiments}) on a high-dimensional system with small instability complexity. Given the same number of learned parameters, the unified predictor attains orders-of-magnitude lower error than spectral filtering, finite-memory, and autoregressive predictors, confirming that it captures the few hard modes without paying for the full state dimension, unlike existing methods.

\section{Related Work}
\label{sec:related}

\textbf{Learning stable dynamical systems.}
The cleanest online learning guarantees for dynamical systems are known in stable or non-explosive regimes. General sequential 
forecasting work studies prediction under temporal dependence and non-stationarity without relying on i.i.d. assumptions \citep{kuznetsov2015learning}.
 For linear dynamical systems, spectral filtering \citep{hazan2017spectral} gives a particularly relevant result: stable LDS outputs can be predicted
  online using a small universal family of filters, with parameter count independent of the hidden-state dimension. \citet{hazan2026spectralfilteringlearningquantum} extend spectral filtering to systems with complex-valued eigenvalues, although their regret scales linearly with the sequence length. \citet{ghai2020no} extend no-regret prediction guarantees to marginally stable systems.
  Related work on learning Kalman filters gives
   regret and sample-complexity guarantees for partially observed linear systems under stochastic and non-explosive assumptions \citep{tsiamis2020online,zhang2023kalman}. 
   More recently, \citet{dogariu2025universal} extended spectral-filtering ideas toward marginally stable nonlinear dynamics. These works show that stability or marginal stability enables 
   compact prediction without full system identification, but they do not address systems with genuinely exploding\footnote{Exploding dynamics do not by themselves preclude bounded mean-square estimation error: classical Kalman filtering can achieve this under standard assumptions \citep[Sec.~4.4]{anderson1979optimal}. The challenge addressed here is compact online prediction, with sublinear regret and only $\widetilde O(\ustb)$ learned parameters, independent of the full hidden-state dimension $\hiddenDim$.}, complex, or non-semisimple modes.

\textbf{Unstable dynamical systems.}
Several works study learning problems in which the underlying dynamics may be
unstable, but mostly in settings different from the partially observed prediction
problem considered here. Classical work of \citet{white1958limiting} studies the
one-dimensional explosive autoregressive process and shows that instability
changes the asymptotic behavior of standard estimators; in this fully observed
scalar setting, the exploding signal can make the parameter easier to estimate.
\citet{faradonbeh2018finite} study finite-time system identification for unstable
linear systems under full-state observation. More broadly,
\citet{simchowitz2018learning,sarkar2019near,jedra2023finite} give finite-sample
identification guarantees for arbitrary or poorly mixing LDSs, targeting recovery
of the system parameters rather than compact prediction from outputs.
\citet{wagenmaker2020active} study active input design for identifying LDSs
efficiently. Thus, while unstable dynamics have been studied in identification
and input-design settings, these works do not provide the partially observed,
parameter-adaptive online prediction guarantees considered here.

\textbf{Controlling linear dynamical systems from data.}
\citet{hazan2022introduction} surveys regret-based formulations connecting online convex optimization to control.
The nonstochastic/adversarial line \citep{agarwal2019online,chen2021blackbox,gradu2023adaptive} studies sublinear regret under adversarial disturbances and time variation; \citet{chen2021blackbox} in particular shows that black-box online control can incur unavoidable exponential-in-dimension overhead, complementing minimax hardness for underactuated plants in \citet{tsiamis2022hard}.
For quadratic costs, \citet{foster2020logarithmic} achieves logarithmic regret when the dynamics are known but disturbances are adversarial, whereas \citet{simchowitz2020naive} establishes that adaptive LQR with unknown models has $\tilde\Theta(\sqrt{T})$ minimax regret---contrasting with the polylogarithmic prediction rates available in some filtering settings.
On the synthesis side, \citet{dean2020sample} couples identification with robust control, and \citet{mania2019certainty} shows certainty-equivalent LQ/LQG controllers inherit particularly favorable (quadratic-in-error) suboptimality scalings when estimates are accurate.


\section{Problem Setup}
\label{sec:setup}

\paragraph{Linear dynamical system.}

We consider a bounded input sequence $u_1,u_2,\ldots \in \mathbb{R}^{\inputDim}$ with $\|u_t\|_2\le K$ for all $t$, and
the corresponding output sequence
$y_1,y_2,\ldots \in \mathbb{R}^{\outputDim}$ generated by an unknown linear
dynamical system described below. The hidden state
$x_t \in \mathbb{R}^{\hiddenDim}$ summarizes the system's memory of past
inputs, with dynamics
\begin{align}
\label{eq:lds}
  x_t \;=\; A\,x_{t-1} + B\,u_t, \qquad y_t \;=\; C\,x_t, \qquad x_0 \;=\; 0,
\end{align}
where $A\in\mathbb{R}^{\hiddenDim\times\hiddenDim}$,
$B\in\mathbb{R}^{\hiddenDim\times\inputDim}$,
$C\in\mathbb{R}^{\outputDim\times\hiddenDim}$.
We assume the pair $(A,C)$ is observable.
We use the shorthand $y_{1:t}:=(y_1,\ldots,y_t)$ and likewise for other
sequences.

\paragraph{Stable modes and instability complexity.}

The eigenvalues $\lambda_1,\dots,\lambda_{\hiddenDim}$ of $A$ in its Jordan
decomposition over $\mathbb{C}$ are counted with algebraic multiplicity. We use
the following complexity measure to isolate the modes which are hardest to learn.

\begin{definition}[Instability complexity]\label{def:ustb}
The \emph{instability complexity} of an LDS, denoted $\ustb$, is the smallest
integer $k\ge 0$ such that at most $k$ eigenvalues of $A$ (counted with
algebraic multiplicity) either lie outside $(-1,1)$ with
$|\lambda_i|>1-1/\log{(k+2)}$, or are non-semisimple. We set $\ustb=1$ when no
eigenvalue satisfies these conditions.
\end{definition}

We also write
\[
  \rho_{\mathrm{st}}:=\max\{|\lambda_i|:\lambda_i\in(-1,1)\text{ and semisimple}\},
\]
with the convention $\rho_{\mathrm{st}}=0$ if there are no such eigenvalues.

Intuitively, the instability complexity counts the modes whose impulse
responses are not guaranteed to lie close to a common low-dimensional
subspace spanned by a few fixed filters. Rapidly decaying semisimple modes
can be approximated using a short input window, and stable real semisimple
modes using a few spectral filters, even when they decay slowly. Other modes
count toward $\ustb$.

\paragraph{\normalcolor Prediction goal.}

Our goal is to design an online algorithm that, given
$(u_t,\,y_{1:t-1})$ at each step, produces a prediction $\pred_t$ whose
average squared loss
\( \frac{1}{T} \sum_{t=1}^T \ell_2^2(\pred_t,y_t) \)
vanishes as $T\to\infty$, where $\ell_2(\pred_t,y_t) = \|\pred_t -y_t \|_2$ denotes
the standard $L_2$ loss.

\paragraph{\normalcolor Controlled inputs.}

Our motivating goal is to learn to stabilize unknown systems.
Prediction is useful for this purpose only when the trajectory remains bounded:
once the system has already exploded, the rollout has left the setting where
predictions can be used to keep it stabilized.
We therefore focus on system-input pairs that keep the dynamics \emph{under
control}: although the open-loop dynamics may be unstable, the chosen inputs
keep the resulting outputs within a fixed bounded region.
The input-generating process need not be known to the learner;
prediction uses only the observed inputs and outputs.
Moreover, systems whose open-loop dynamics are already strictly
stable (including those with complex eigenvalues) are inherently under control
for bounded inputs; our algorithm naturally handles this setting as well\footnote{Conversely, systems that are not controlled rapidly explode. Such trajectories are statistically benign: the output sequence becomes dominated by the system's largest mode, typically making prediction easy.}.
This is formally defined as follows:

\begin{assumption}[Controlled outputs]\label{ass:controlled-outputs}
There exists a constant $B_y<\infty$, independent of the horizon $T$, such that
for the LDS $(A,B,C)$ and the input sequence $(u_t)_{t\ge 1}$ under
consideration, the resulting outputs satisfy
\[
  \max_{1\le t\le T}\|y_t\|_2 \le B_y
  \qquad\text{for every }T\ge 1.
\]
We call such an input sequence \emph{$B_y$-controlling} for the system.\footnote{Bounded inputs can produce growing outputs in marginally stable systems, for example through resonance. Our assumption excludes such trajectories. We leave extending the result to this case to future work.}
\end{assumption}

One motivating setting is learning a system while a stabilizing
operator keeps it bounded, as when a pilot flies an open-loop unstable
helicopter. Stabilization and control optimization can be pursued jointly,
with the goal of improving control without losing stability. Prediction
remains useful because, even after the unstable component is controlled,
the remaining dynamics may determine the accumulated cost. Using the learned
predictor to construct a stabilizing controller is left to future work.

\section{Algorithm and Main Result}

\label{sec:main}
\subsection{Predictor Classes}
\label{sec:classes}
We begin by describing the predictor class from which our unified algorithm
(\cref{alg:unified}) will learn online. The central goal of this section is
to construct, in \cref{sec:unified-class}, a predictor class that can
approximate \emph{any} LDS to arbitrary accuracy using only
$O(\ustb\cdot\mathrm{polylog}(T, \ustb))$ parameters, a quantity that depends linearly in $\ustb$ 
and only logarithmically in $\ustb$ and $T$.

Our construction draws on three well-known predictor classes. The first,
\emph{spectral filtering} (\cref{sec:sf-approx}), approximates stable
real-diagonalizable LDS outputs using a number of parameters independent of the
state dimension, but fails on systems with non-stable or complex modes.
The second, \emph{finite-memory input filters} (FIR) (\cref{sec:fir-approx}), handle
modes whose impulse responses decay quickly. The third,
\emph{autoregressive} (AR) \emph{predictors} (\cref{sec:ar-approx}), represent
any observable LDS exactly, but at the cost of a parameter count proportional to
the full state dimension. None of these components alone suffices for our
setting: we want a predictor that captures every LDS, including non-diagonalizable systems and systems with
exploding or complex eigenvalues, while paying a parameter cost only in the
instability complexity. Combining the three families gives precisely such a class.

\subsubsection{Spectral filters for stable real-diagonalizable LDS}
\label{sec:sf-approx}

For stable real-diagonalizable LDS (equivalently, all modes are real,
semisimple, and lie in $(-1,1)$), spectral filtering
\citep{hazan2017spectral} approximates the output arbitrarily well
using a small, fixed set of universal basis filters---independent of the
system parameters---applied to the input history. Write
$\tilde u_t:=(u_t,u_{t-1},\dots,u_1,0,\dots,0)\in\mathbb{R}^{\inputDim\times T}$
for the zero-padded input history, and let $\phi_1,\dots,\phi_T\in\mathbb{R}^T$
denote the universal spectral filters as defined in \cref{def:spectral-filters}.
The following proposition follows from the spectral filtering framework
\citep{hazan2017spectral,hazan2022introduction}.

\begin{proposition}[Informal; see \cref{lem:sf-diagonal-real} for the formal version]%
\label{thm:sf-approx}
Consider a stable real-diagonalizable LDS as in \cref{eq:lds}. For every
$\varepsilon>0$, there exist $h = O\!\left(\log T\cdot\log\frac{T}{\varepsilon}\right)$
and matrices $M_1,\dots,M_h\in\mathbb{R}^{\outputDim\times\inputDim}$
depending only on $(A,B,C)$, such that
\[
  \sum_{t=1}^{T}\Bigl\|y_t - \sum_{s=1}^{h} M_s(\tilde u_t\phi_s)\Bigr\|^2
  \;\le\; \varepsilon.
\]
In particular, $h$ has no polynomial dependence on the state dimension $\hiddenDim$.
\end{proposition}

\subsubsection{Finite-memory filters for rapidly decaying modes}
\label{sec:fir-approx}

Finite-memory input filters approximate the output using only a short window of
past inputs. They apply whenever the relevant modes are semisimple and strictly
stable, as formalized in the following theorem.

\begin{proposition}[Informal; see \cref{lem:fast-fir} for the formal version]%
\label{thm:fir-approx}
Consider an observable LDS as in \cref{eq:lds} whose state matrix $A$ is
diagonalizable and satisfies $|\lambda_i|\le 1-\tfrac{1}{L}$ for every eigenvalue
$\lambda_i$. For every $\varepsilon>0$, there exist
$L_\varepsilon=O\!\left(L\log\tfrac{T}{\varepsilon}\right)$ and matrices
$Q_0,\dots,Q_{L_\varepsilon-1}\in\mathbb{R}^{\outputDim\times\inputDim}$,
depending only on $(A,B,C)$, such that
\[
  \sum_{t=1}^{T}\Bigl\|y_t-\sum_{\ell=0}^{L_\varepsilon-1}Q_\ell\,u_{t-\ell}\Bigr\|^2
  \;\le\; \varepsilon.
\]
In particular, $L_\varepsilon$ is independent of the state dimension
$\hiddenDim$.
\end{proposition}

\subsubsection{Autoregressive predictors for general LDS}
\label{sec:ar-approx}

For any observable LDS with state dimension $\hiddenDim$, the output satisfies an
exact autoregressive relation: the last $\hiddenDim$ past outputs and the $\hiddenDim$ most recent inputs
determine $y_t$ exactly. Unlike spectral filtering and FIR, this representation
requires no approximation, but its parameter count is proportional to $\hiddenDim$.
This exact representation is a standard consequence of realization theory and the Cayley--Hamilton theorem
\citep{hazan2022introduction,kailath1980linear,cayley1858memoir}

\begin{proposition}[Informal; see \cref{lem:ar-only} for the formal version]%
\label{thm:ar-approx}
Consider any observable LDS as in \cref{eq:lds} with state dimension $\hiddenDim$.
Then there exist coefficients $\alpha_1,\dots,\alpha_{\hiddenDim}\in\mathbb{R}$ and
matrices $N_0,\dots,N_{\hiddenDim-1}\in\mathbb{R}^{\outputDim\times\inputDim}$
depending only on $(A,B,C)$, such that for all $t\ge \hiddenDim+1$,
\[
  y_t
  \;=\;
  \sum_{i=1}^{\hiddenDim}\alpha_i\,y_{t-i}
  \;+\;
  \sum_{i=0}^{\hiddenDim-1}N_i\,u_{t-i}.
\]
\end{proposition}

\subsubsection{The unified predictor class}
\label{sec:unified-class}
In this section we introduce a hybrid predictor class that combines the
strengths of the spectral filtering, finite-memory filtering, and
autoregressive predictor classes introduced in the previous sections.
While it might seem that these classes can be trivially combined,
the difficulty is that their approximation errors interact. The AR block acts
on the observed outputs with coefficients that can be exponentially large in
$\ustb$, amplifying the errors of the other two components. We therefore need
exponentially more accurate approximations and a learner whose regret depends
only logarithmically on the coefficient norms.

Below we prove that any dynamical system
with instability complexity $\ustb$ can be approximated using only
$O\!\left(\ustb\cdot\mathrm{polylog}(T, \ustb, 1/\varepsilon)\right)$ parameters,
depending only on $\ustb$ and not on the dimension of the state space.

\begin{theorem}[Informal; see \cref{lem:ar-plus-sf} for the formal version]%
\label{thm:uf-approx-informal}
Consider any LDS as in \cref{eq:lds} with instability complexity $\ustb$. For
every $\varepsilon>0$, there exist
$h = O\!\left(\ustb\cdot\log T\cdot\log\frac{T \cdot \ustb}{\varepsilon}\right)$ and
$L = O\!\left(\ustb\cdot\mathrm{polylog}(T,\ustb,1/\varepsilon)\right)$,
coefficients $\alpha_1,\dots,\alpha_{\ustb}\in\mathbb{R}$, and matrices
$W_s,Q_\ell\in\mathbb{R}^{\outputDim\times\inputDim}$, such that
\[
  \pred_t
  \;=\;
  \sum_{i=1}^{\ustb}\alpha_i\,y_{t-i}
  \;+\;
  \sum_{s=1}^{h}W_{s}(\tilde u_{t}\phi_s)
  \;+\;
  \sum_{\ell=0}^{L+\ustb}Q_\ell\,u_{t-\ell}
\]
satisfies $\sum_{t=1}^{T}\|y_t-\pred_t\|^2 \le \varepsilon$, using
$O\!\left(\ustb\cdot\mathrm{polylog}(T, \ustb, 1/\varepsilon)\right)$ parameters.
\end{theorem}

\begin{proof}[Proof sketch (full proof in Appendix~\ref{app:ar-plus-sf})]
Write $y_t = y_t^{(\mathrm{st})} + y_t^{(\mathrm{fast})} + y_t^{(\mathrm{ar})}$
for the contributions of, respectively, the real semisimple modes with
$|\lambda|>1-1/\log(\ustb+2)$ (the \emph{spectral group}), the remaining
semisimple modes (all strictly stable, with modulus at most
$1-1/\log(\ustb+2)$), and the $\ustb$ modes counted by the instability
complexity.
By Cayley--Hamilton applied to the AR block, the third piece
$y_t^{(\mathrm{ar})}$ satisfies an exact AR relation of order at most $\ustb$.
Since $y_t^{(\mathrm{ar})}$ is not directly observed, we substitute
$y_{t-i}^{(\mathrm{ar})} = y_{t-i}-y_{t-i}^{(\mathrm{st})}-y_{t-i}^{(\mathrm{fast})}$,
rewriting $y_t$ as observable AR lags of $y$ plus the AR residuals of the
stable and fast components.
The AR residual of the stable component can be approximated, at time $t$, by
spectral filters applied to the current input window together with a few of the
most recent inputs $u_t,\dots,u_{t-\ustb}$. The fast residual is approximated by a single FIR convolution of length
$L+\ustb$. Young's inequality then bounds how the FIR error propagates through
the AR coefficients, yielding the theorem.
\end{proof}

\subsection{The Unified Algorithm}
\label{sec:algo}

We now turn the approximation results of \cref{sec:classes} into an online
learning algorithm. The goal is to compete with an unknown LDS
without knowing its matrices or spectrum, while using only
$O(\ustb\cdot\mathrm{polylog}(T, \ustb))$ parameters.
This is the sense in which the
algorithm is a unified one: the same procedure applies to stable systems, systems
with complex eigenvalues, and systems with exploding modes.

The algorithm learns the coefficients of the unified predictor class from
\cref{sec:unified-class}. At time $t$, the feature vector contains the
$\ustb$ autoregressive output lags, a finite-memory input window of length
$L+\ustb+1$, and the spectral-filter features of the current input window. We
run one copy of the Vovk--Azoury--Warmuth forecaster per output coordinate; see
Appendix~\ref{app:vaw} for the definition of the algorithm and its regret guarantee.

In each iteration, the VAW forecaster for the $i$-th output coordinate receives a feature vector $a_{t,i}$ and is asked to make
a prediction $\widehat{y}_t^{(i)} = a_{t,i}^\top \widehat{x}_t$. The algorithm then observes the true
output $y_t^{(i)}$ and suffers squared loss $(\widehat{y}_t^{(i)} - y_t^{(i)})^2$.
The feature vectors at time $t$ for the $i$-th output coordinate are defined as follows:
\[
  a_{t,i}
  \;=\;
  \begin{bmatrix}
    Y_{t,i}^\top & F_t^\top & \tilde U_t^\top
  \end{bmatrix}^{\top}
  \in\mathbb{R}^{\hiddenDim_{\mathrm{feat}}},
\]
where $\hiddenDim_{\mathrm{feat}} = \ustb + \inputDim(L+\ustb+1) + h\inputDim$
and, for $i\in[\outputDim]$,
\begin{align}\label{item:Yt-Ut-tildeUt}
  Y_{t,i} &= \begin{bmatrix}(y_{t-1})_i & \cdots & (y_{t-\ustb})_i\end{bmatrix}^{\top}\in\mathbb{R}^{\ustb}, \\
  F_{t} &= \begin{bmatrix}u_t^\top & \cdots & u_{t-(L+\ustb)}^\top\end{bmatrix}^{\top}\in\mathbb{R}^{\inputDim(L+\ustb+1)}, \\
  \tilde U_{t} &= \begin{bmatrix} (\tilde{u}_{t}\phi_1)^\top & \cdots &  (\tilde{u}_{t}\phi_h)^\top\end{bmatrix}^{\top}\in\mathbb{R}^{h\inputDim}.
\end{align}
The window $F_t$ subsumes the $\ustb$ recent inputs originally used by the AR
component and additionally carries the lags needed by the finite-memory
predictor for strictly stable modes. These features are exactly the parameters
of a predictor in the unified class of \cref{thm:uf-approx-informal}.

\begin{algorithm}[H]
\caption{Unified LDS prediction.}
\label{alg:unified}
\begin{algorithmic}[1]
  \State \textbf{Input:} AR order $\ustb$, filter count $h$, finite-memory length $L$.
  \State Initialize VAW instances $\alg_1,\dots,\alg_{\outputDim}$ with dimension
         $\hiddenDim_{\mathrm{feat}}$.
  \For{$t=1,2,\dots,T$}
    \State Observe $u_t \in \mathbb{R}^{\inputDim}$.
    \State Set vectors $Y_{t,i}, F_t, \tilde{U}_t, \forall i \in [\outputDim]$ as defined in \cref{item:Yt-Ut-tildeUt}.
    \State Set $a_{t,i} \;\gets\; 
      \begin{bmatrix}Y_{t,i}^\top & F^\top_{t} & \tilde{U}_t^\top\end{bmatrix}^{\top}$.
    \State For each $i\in[\outputDim]$, obtain $m_t^i \gets \alg_i.\mathrm{predict}(a_{t,i})$.
    \State Predict $\pred_t \;\gets\; (m_t^1,\dots,m_t^{\outputDim})^{\top}$.
    \State Observe $y_t$ and suffer loss
           $\ell_2^2(\pred_t,y_t) = \|\pred_t-y_t\|_2^2$.
    \State For each $i\in[\outputDim]$, update
           $\alg_i.\mathrm{update}\!\big((y_t)_i\big)$.
  \EndFor
\end{algorithmic}
\end{algorithm}

\begin{theorem}[Informal; see \cref{thm:main-regret-formal} for the formal version]\label{thm:main-regret}
Consider any observable LDS with state dimension $\hiddenDim$,
of which the instability complexity is $\ustb$,
and any input sequence whose generated outputs satisfy
\cref{ass:controlled-outputs}. With filter count
$h = O\!\left(\ustb \cdot \log T\cdot\log(\mathrm{poly}(T,\ustb))\right)$,
finite-memory length
$L = O\!\left(\ustb\cdot\mathrm{polylog}(T,\ustb)\right)$,
\cref{alg:unified} achieves loss
\[
  \frac{1}{T}\sum_{t=1}^{T}\ell_2^2(\pred_t,y_t) = O\left( \frac{\ustb^2 \log^3(\ustb) \cdot \log^3(T)}{T} \right)
\]
assuming the parameters of the LDS are constants except $T, \ustb$\footnote{As in prior work \citep{hazan2017spectral,ghai2020no}, we hold the input and output dimensions and the bounds on inputs, outputs, and system-matrix norms fixed as $T,\ustb,\hiddenDim$ vary. The dependence on these quantities is explicit in \cref{lem:ar-plus-sf,lem:vaw-regret}.}, and uses $O(\ustb\cdot\mathrm{polylog}(T, \ustb))$ learnable parameters.
\end{theorem}

\begin{proof}[Proof sketch (full proof in Appendix~\ref{app:main-regret-proof})]
We decompose the cumulative loss into two parts. Let $h^{\star}$ be the
value of $h$ from the statement and let $f^{\star}$ be the predictor from
\cref{thm:uf-approx-informal} with that $h^{\star}$.
\begin{align*}
  \sum_{t=1}^T \ell_2^2(\pred_t,y_t)
  &\;=\;
  \underbrace{\Big(\sum_t \ell_2^2(\pred_t,y_t)
    -\sum_t \ell_2^2(f^{\star}_t,y_t)\Big)}_{(I)\text{: VAW regret}} +
  \underbrace{\sum_t \ell_2^2(f^{\star}_t,y_t)}_{(II)\text{: approximation error}}.
\end{align*}
The first term, $(I)$, is the estimation error of the chosen predictor class:
it is the regret of the per-coordinate VAW forecaster on the features $a_{t,i}$
defined in \cref{alg:unified}. Since the features $a_{t,i}$, the comparator class from
\cref{lem:ar-plus-sf}, and the labels $y_t^{(i)}$ are bounded by constants, the standard VAW analysis
(see \cref{lem:vaw-regret} in Appendix~\ref{app:vaw}) gives logarithmic regret in $T$.

The second term, $(II)$, is the approximation error: it measures
how well the class can approximate the real output sequence. This term is
controlled by \cref{lem:ar-plus-sf}, which gives that $f^{\star}$
achieves total squared error
$\sum_t \ell_2^2(f^{\star}_t,y_t)\le\varepsilon$ against $y_{1:T}$. Choosing
the optimal $\varepsilon$ and combining the two terms yields the claimed $o(T)$ regret. 
\end{proof}

\subsubsection{Extensions}

The theorem above assumes a correctly specified spectral cutoff and
known $T$ and $\ustb$; the following extensions address cutoff
misspecification and remove the need to know these two parameters.

\textbf{Robustness to the spectral cutoff.}
Allowing the finite-memory block's spectral radius to exceed its nominal
cutoff by a sufficiently small $\delta$ adds only $O(\delta)$ to the
cumulative loss guarantee, for fixed approximation accuracy and problem
parameters, including $\ustb$. The predictor size and VAW regret bound
are unchanged (\cref{prop:cutoff-robustness}).

\textbf{Adapting to unknown horizon and complexity.}
The horizon $T$ and instability complexity $\ustb$ need not be known in
advance, given the known input, output, and system bounds specified in
Appendix~\ref{app:unknown-parameters}. This costs only an additional $\log(T)$ factor in the
loss bound while retaining $\widetilde O(\ustb)$ learnable parameters
(see \cref{thm:unknown-parameters} for details).

\section{Lower Bounds and Motivation}
\label{sec:lower}

Our unified algorithm, \cref{alg:unified}, has learnable parameter count
$\widetilde O(\ustb)$, depending on the instability complexity rather
than the full state dimension. In the following subsections we explain why
$\ustb$ is the correct parameter to depend on. First, we show that
in a broad diagonalizable class, many exploding modes relative to the
number of input channels create a practical obstruction to stabilization,
motivating the small-$\ustb$ regime under limited actuation
(\cref{sec:lower-cond}). Second, we show that at least a linear dependence on $\ustb$ is
unavoidable for filter-based predictors (\cref{sec:lower-filter}).

\subsection{Systems with many exploding modes and few inputs are practically uncontrollable}
\label{sec:lower-cond}

Our motivating goal is to stabilize systems whose open-loop dynamics may be
unstable. The results above show that prediction---a natural prerequisite for
control---can be done with a memory footprint depending only on $\ustb$, the
instability complexity, rather than the full state dimension $\hiddenDim$.
We now argue that a dependence on $\ustb$ is a minor concession
when actuation is limited: already in a broad diagonalizable class,
many exploding modes relative to the number of input channels can make
stabilization impractical.

To formalize this, we introduce the \emph{input dynamic range} of a control sequence $u$: the ratio of the largest to the smallest nonzero input magnitude, denoted $\IDR(u)$. A large input dynamic range means the actuator must simultaneously operate at vastly different scales---an extreme sensitivity to rounding, actuator limits, and noise. 

\begin{definition}\label{def:idr}
For a nonzero input sequence $u=(u_t)_{t\ge 1}$ with $u_t\in\mathbb R^{\inputDim}$, its \emph{dynamic range}, denoted $\IDR(u)$, is defined as follows:
\[
  \IDR(u) \;:=\; \frac{\sup_{t\ge 1}\|u_t\|_2}{\inf\{\|u_t\|_2:u_t\ne 0\}},
\]
with $\IDR(u)=\infty$ if the numerator is infinite or the denominator is zero.
\end{definition}

We show this obstruction in a simple diagonalizable setting with
$\ustb$ exploding modes and $\inputDim$ input channels. With high probability
over a uniformly random first nonzero input direction, every stabilizing
continuation has dynamic range growing exponentially in $\ustb/\inputDim$
(\cref{prop:cond-number}). For a single input channel, the conclusion holds
for every nonzero stabilizing sequence.
We state the result in diagonal coordinates; any diagonalizable system with
distinct real eigenvalues can be put in this form by a change of basis.

\begin{proposition}\label{prop:cond-number}\label{prop:multi-input-cond-number}
Let $1\le\inputDim\le\ustb$ and
$A=\operatorname{diag}(\lambda_1,\dots,\lambda_\ustb)\in\mathbb R^{\ustb\times\ustb}$
have distinct real entries satisfying $|\lambda_i|\ge1+\eta$ for some
$\eta>0$. Let $B\in\mathbb R^{\ustb\times\inputDim}$ have no zero
rows,\footnote{This excludes input-decoupled modes, which can be ignored.}
and consider
\[
  x_t=Ax_{t-1}+Bu_t,\qquad x_0=0.
\]
For any $\delta\in(0,1)$, draw $v$ uniformly from the unit sphere
$\mathbb S^{\inputDim-1}$. With probability at least $1-\delta$, every nonzero
input sequence whose first nonzero input is a positive multiple of $v$ and
which stabilizes the system, meaning $x_t\to0$, satisfies
\begin{equation}\label{eq:idr-multiple-input}
  \IDR(u)\ge
  \frac{\delta}{\sqrt{\inputDim}}\,
  \frac{\eta}{2(1+\eta)}
  \left(1+\frac{\eta}{2}\right)^{\ustb/\inputDim}.
\end{equation}
\end{proposition}

\begin{proof}[Proof sketch (full proof in Appendix~\ref{app:multi-input-idr})]
We construct a positive definite Gramian $G$ from the system matrices $A$
and $B$, and show that $B^\top G^{-1}B$ has an eigenvalue that grows
exponentially in $\ustb/\inputDim$. Stabilization requires subsequent
controls to cancel the effect of the first nonzero control $u_1$.
Using this eigenvalue bound, we show that, with high probability over a
uniformly random direction of $u_1$, every stabilizing continuation must
contain a control whose norm is exponentially larger than that of $u_1$.
This gives the claimed lower bound on the input dynamic range.
\end{proof}

\subsection{Filter-count lower bound}
\label{sec:lower-filter}

Our algorithm uses a filter class whose size grows linearly with $\ustb$ up
to logarithmic factors. We now show that this linear dependence is unavoidable for
predictors of the same broad type, even when the filters themselves are chosen
arbitrarily. The lower bound applies to bounded-coefficient filter predictors:
predictors that form bounded linear combinations of fixed causal filters applied
to past inputs and outputs. For such predictors, fewer than $\ustb$ filter
directions cannot guarantee sublinear regret against the LDS class. Since
\cref{sec:lower-cond}  motivates the regime of small
$\ustb$, the resulting dependence is not on the hidden dimension or an
expressivity obstruction.

We formalize this by defining a general filter class (\cref{def:guf}) that
applies arbitrary fixed $\ell_1$-bounded filters to past outputs and inputs,
with a bound on the total learned coefficient mass. Our unified predictor
(\cref{thm:uf-approx-informal}) is an instance of this template: its output
filters are delay operators, and its input filters consist of spectral filters
and finite-memory delay operators.
We then show (\cref{prop:filter-lb}) that if the total filter count is below
$\ustb$, there is an LDS whose output no predictor in the class can track with
vanishing average loss.

\begin{definition}[General filter class]\label{def:guf}
Let $\Psi^y=\{\psi^y_1,\dots,\psi^y_{q_y}\}\subset\ell_1(\mathbb{N})$ be
strictly causal output filters with $\|\psi^y_r\|_1\le 1$, and let
$\Psi^u=\{\psi^u_1,\dots,\psi^u_{q_u}\}\subset\ell_1(\mathbb{N}_0)$ be input
filters with $\|\psi^u_a\|_1\le 1$. For scalar input and output sequences,
write
\[
  (\psi^y_r * y)_t \;:=\; \sum_{\tau=1}^{t-1}\psi^y_{r,\tau}\,y_{t-\tau},
  \qquad
  (\psi^u_a * u)_t \;:=\; \sum_{\tau=0}^{t-1}\psi^u_{a,\tau}\,u_{t-\tau}.
\]
For $R<\infty$, the associated bounded-coefficient filter predictor class is
\[
  \GUFcl_{\Psi^y,\Psi^u}^{R}
  \;:=\;
  \left\{\,f\;:\;
  \begin{aligned}
    f_t &= \textstyle\sum_{r=1}^{q_y}\alpha_r(\psi^y_r*y)_t
       + \textstyle\sum_{a=1}^{q_u}\beta_a(\psi^u_a*u)_t,\\
    &\textstyle\sum_{r=1}^{q_y}|\alpha_r|
       + \textstyle\sum_{a=1}^{q_u}|\beta_a| \le R
  \end{aligned}
  \,\right\}.
\]
We call $q_y+q_u$ the \emph{total filter count}.
\end{definition}

\begin{proposition}\label{prop:filter-lb}
Fix $\ustb\ge 1$ and $R<\infty$, and let $\Psi^y,\Psi^u$ be any filters as in
\cref{def:guf} with total filter count $q_y+q_u<\ustb$. Then there exist a
$\ustb$-dimensional real LDS $(A,B,C)$ with
$\|A\|_\infty,\|B\|_\infty,\|C\|_\infty\le 1$, and a bounded input sequence
$(u_t)_{t\ge 1}$ that is $1$-controlling such that the resulting outputs
satisfy
\[
  \liminf_{T\to\infty}\;
  \inf_{f\in\GUFcl_{\Psi^y,\Psi^u}^{R}}
  \frac{1}{T}\sum_{t=1}^{T}\ell_2^2(f_t,y_t)
  \;>\; 0.
\]
\end{proposition}

\begin{proof}[Proof sketch (full proof in Appendix~\ref{app:filter-lb-proof})]
We prove this directly by explicitly constructing such an LDS\footnote{Our construction is heavily influenced by the lower bound argument in \citet{hazan2026spectralfilteringlearningquantum}.}. The construction 
uses a standard delay-line system: an input impulse enters the first state
coordinate, shifts one coordinate at a time, and is observed as output only after 
a specifically chosen delay. Since the predictor has strictly fewer than $\ustb$ 
input-filter directions, linear algebra dictates that at least one of the $\ustb$ 
possible delays cannot be represented by the first $\ustb$ coefficients of these filters.
We then feed the system a sequence of isolated impulses separated by long gaps. 
The $\ell_1$ coefficient bound ensures that contributions from old impulses 
decay and become negligible.
Thus, in every block, the predictor makes a constant error on the unrepresentable 
delay, giving a strictly positive average squared loss.
\end{proof}

\section{Experimental Validation}
\label{sec:experiments} 

We empirically test the prediction suggested by \cref{thm:main-regret}: for a
high-dimensional LDS with small instability complexity, the unified predictor
should learn with a parameter budget tied to $\ustb$ rather than to the full
state dimension. We construct a $503$-dimensional controlled system with
$\ustb=3$, train several online linear predictors with the same number of
learned scalar parameters using the same Vovk--Azoury--Warmuth forecaster, and
compare their normalized mean-squared error\footnote{We normalize the system
by a scalar so that the zero predictor has normalized mean-squared error $1$.}
on a single training
trajectory of length $T=200{,}000$. To measure the asymptotic behavior, we report
the average loss over the last $10{,}000$ time steps. More extensive details for
the experimental setting are in
Appendix~\ref{app:experimental-details}. Code for reproducing all
demonstrations can be found in \url{https://github.com/YuvMilo/UnifiedLearningLDS}.

\begin{table}[H]
  \centering
  \caption{Controlled $\ustb=3$ system with equal parameter budget. We report the
  best and worst normalized mean-squared error over five
  random seeds, measured on the last $10{,}000$ time steps.}
  \label{tab:experimental-validation}
  \begin{tabular}{lcc}
    \toprule
    Predictor class & Best seed & Worst seed \\
    \midrule
    Finite-memory input filters & $3.76\cdot 10^{-1}$ & $9.60\cdot 10^{-1}$ \\
    Autoregressive predictor & $1.73\cdot 10^{-5}$ & $9.08\cdot 10^{-5}$ \\
    Spectral filtering & $5.20\cdot 10^{-3}$ & $6.13\cdot 10^{-2}$ \\
    Unified predictor & $7.88\cdot 10^{-9}$ & $1.65\cdot 10^{-8}$ \\
    \bottomrule
  \end{tabular}
\end{table}

\paragraph{System.}
The system is scalar-input, scalar-output, and has state dimension
$\hiddenDim=503$. It contains one exploding scalar mode with eigenvalue $1.3$
and one slow complex conjugate pair with eigenvalues $0.98e^{\pm i1.57}$; these
three eigenvalues are the only modes counted by the instability complexity. 
The other modes are $300$ real stable modes in $(-1,1)$ and $100$ fast-decaying
complex conjugate pairs of the form $r e^{\pm i\theta}$ with
$r\in[0,0.25]$.
The real stable part is chosen so that its input-output response has large
components in many spectral-filter directions, making it difficult to learn from
a short input window alone. The input controls the exploding mode explicitly. We
draw independent signs $z_t\in\{-1,+1\}$ along the trajectory and set
$u_t=z_t-1.3z_{t-1}$, with $z_0=0$. Thus the open-loop system is unstable, but the observed trajectory
remains bounded.

\paragraph{Predictor classes.}
The comparison is an ablation study of the three ingredients in the unified
predictor.
To make the comparison fair, each
predictor has exactly $16$ learned scalar parameters and is trained by the same
Vovk--Azoury--Warmuth learner from Appendix~\ref{app:vaw}. The finite-memory
input-filter predictor spends all $16$ parameters on recent inputs, as in
\cref{sec:fir-approx}. The autoregressive predictor splits the budget evenly
between $8$ output lags and $8$ input lags, as in \cref{sec:ar-approx}. The
spectral-filtering predictor spends all $16$ parameters on spectral filters, as
in \cref{sec:sf-approx}. The unified predictor uses all three components: $3$
output lags, $7$ input lags, and $6$ spectral filters, matching the construction
in \cref{sec:unified-class}.

\paragraph{Results.}
\Cref{tab:experimental-validation} reports the final normalized mean-squared
error range over five random seeds. The unified predictor gives the lowest
error by a wide margin, showing that the advantage comes from combining the
autoregressive, finite-memory, and spectral-filtering features.

\section{Limitations and Future Work}
\label{sec:limitations}

Several limitations of this work suggest natural directions for future research. First, while our upper and lower bounds match up to polylogarithmic factors in their dependence on $\ustb$ for parameter count, they do not yet settle the optimal logarithmic dependence or the sharp regret constants. Second, our results concern prediction rather than control. A natural next step is to use the predictor as a component of a closed-loop stabilization policy, and to understand how its regret guarantees translate into guarantees for control objectives.

\section{Conclusion}
\label{sec:conclusion}

Motivated by the goal of stabilizing unknown linear dynamical systems from observations, we studied the online prediction problem that precedes it. The existing theory leaves a fundamental gap: available methods either rely on restrictive assumptions on the dynamics, or else pay memory proportional to the full hidden-state dimension. We identified a new regime that captures an important middle ground: systems with only $\ustb$ modes outside the real stable interval that do not decay rapidly, together with non-semisimple modes, embedded in an arbitrarily large stable real semisimple component. We show that this regime is central not only because it is broader than the classical stable setting, but also because  it is strongly motivated by an obstruction to control: already in a broad diagonalizable class, many exploding modes relative to the number of input channels can force stabilizing inputs to span exponentially many scales (\cref{prop:cond-number}), making such dynamics practically uncontrollable due to acute sensitivity to finite precision, actuator limits, and noise.

Within this important regime, we showed that prediction does not need to pay for the full hidden dimension. Our algorithm combines dimension-free spectral filtering for stable real semisimple modes, finite-memory input filtering for fast decaying semisimple modes, and a small autoregressive correction for the $\ustb$ exceptional modes. This yields a unified online predictor for general LDSs that achieves sublinear regret with a learnable parameter count of $\widetilde O(\ustb)$, completely independent of the full hidden-state dimension. The lower bound of \cref{prop:filter-lb} shows that at least linear dependence on $\ustb$ is unavoidable for a fixed filter family covering the LDS class, suggesting the \emph{instability complexity} $\ustb$ as a natural complexity measure for this problem.

We hope that this perspective serves as a foundation for future work that turns such prediction guarantees into closed-loop stabilization methods.

\section*{Acknowledgments}
This work was supported by the European Research Council (ERC) grant NN4C 101164614, a Google Research Scholar Award, a Google Research Gift, Meta, the Yandex Initiative in Machine Learning, the Israel Science Foundation (ISF) grant 1780/21, the Tel Aviv University Center for AI and Data Science, the Adelis Research Fund for Artificial Intelligence, Len Blavatnik and the Blavatnik Family Foundation, and Amnon and Anat Shashua. Elad Hazan gratefully acknowledges support from the Office of Naval Research, Open Philanthropy, and ARIA.


\bibliographystyle{plainnat}
\bibliography{references}

\newpage
\appendix

\section{Spectral filtering approximation: stable real-diagonalizable case}
\label{app:sf-symmetric}

In this section we prove the approximation guarantee of the spectral
filtering class for stable real-diagonalizable LDS, restated as
\cref{thm:sf-approx} in the main text.

\begin{theorem}\label{thm:approx-spectral-filtering}
  If we set the number of filters to be
  \[
    h(\varepsilon, T) = 2 \cdot \log{T} \cdot \log{\left(\frac{C_0 \cdot T^6}{\varepsilon} \right)}
  \]
  where $C_0$ is a constant depending on $\lVert B\rVert_F, \lVert C\rVert_F, K$ and the norm of the matrix $V$, where $A = VDV^{-1}$ and $D = \operatorname{diag}(\lambda_1,\dots,\lambda_{\hiddenDim})$,
   then the best policy from $\mathrm{SF}_{h(\varepsilon, T)}$ class achieves $\ell_2$ loss at most $\varepsilon$. That is, 
  \[
    \sum_{t=1}^T\lVert \pred_t - y_t\rVert_2^2 \leq \varepsilon.
  \]
\end{theorem}
Before we jump into the proof, we will need to introduce some machinery.
\begin{definition}[Spectral Filters]
    \label{def:spectral-filters}
    Suppose $\mu_{\alpha} = \begin{bmatrix} 1 & \alpha & \alpha^2 & \cdots & \alpha^{T-1} \end{bmatrix}^\top$ for $\alpha \in (-1,1)$. Define matrix $Z_{\alpha}$ as:
    \[
        Z_{\alpha} = \int_{-1}^1 \mu_{\alpha} \mu_{\alpha}^\top d\alpha \in \mathbb{R}^{T \times T}
    \]
    Then the spectral filters are the orthonormal eigenvectors $\phi_1, \phi_2, \cdots, \phi_T$ of $Z_{\alpha}$ and the corresponding eigenvalues are $\lambda_1, \lambda_2, \cdots, \lambda_T$ of $Z_{\alpha}$ where $1 >|\lambda_1| \geq |\lambda_2| \geq \cdots \geq |\lambda_T| > 0$.
\end{definition}

We make use of the following two theorems from the book \cite{hazan2022introduction}:
\begin{theorem}
    \label{thm:decay-eigenvalues}
    We have that for all $k \in [T]$
    \[
        \lambda_k \leq 4\pi^2 e^{-\frac{k}{\log{T}}}
    \]
\end{theorem}
\begin{theorem}
    \label{thm:decay-eigenvectors}
    We have that for all $i \in [T]$ and for all $\alpha \in (-1,1)$:
    \[
        |\phi_i^\top \mu_\alpha| \leq \left( 4T^2 \lambda_i\right)^{\frac{1}{4}} \overset{\text{Thm \ref{thm:decay-eigenvalues}}}{\implies} |\phi_i^\top \mu_\alpha| \leq  2\sqrt{\pi} \cdot T^{\frac{1}{2}} \cdot  e^{-\frac{i}{4\log{T}}}
    \]
\end{theorem}

The proof of \cref{thm:approx-spectral-filtering} will be to essentially show that there exists a predictor in the class $\mathrm{SF}_{h}$, where $h = O\left(\log T \cdot \log (T/\varepsilon)\right)$, which $\varepsilon$-approximates any stable real-diagonalizable LDS. Below we provide the full proof:
\begin{proof}[Proof of \cref{thm:approx-spectral-filtering}]
    \label{proof:approx-spectral-filtering}
    Notice that since we are in the stable real-diagonalizable regime, where we can write:
    \[
        A = V D V^{-1} = V \begin{bmatrix}
            \lambda_1 & 0 & \cdots & 0 \\
            0 & \lambda_2 & \cdots & 0 \\
            \vdots & \vdots & \ddots & \vdots \\
            0 & 0 & \cdots & \lambda_d
        \end{bmatrix} V^{-1}
    \]
    where $1 > |\lambda_1| \geq |\lambda_2| \geq \dots \geq |\lambda_d| > 0$. We have that:
    \begin{align*}
        y_t &= \sum_{\tau=0}^{t-1} CA^{\tau}Bu_{t - \tau} = \sum_{\tau=0}^{t-1} CV D^{\tau}V^{-1} Bu_{t - \tau} = \sum_{\tau=0}^{t-1} CV \left( \sum_{i=1}^\hiddenDim \lambda_i^{\tau} e_i e_i^\top\right) V^{-1} Bu_{t - \tau} \\
            &= \sum_{\tau=0}^{t-1} \sum_{i=1}^\hiddenDim \lambda_i^{\tau} (CV e_i) (e_i^\top V^{-1} B)u_{t - \tau} = \sum_{i=1}^\hiddenDim c_i b_i^\top \sum_{\tau=0}^{t-1} \lambda_i^{\tau} u_{t - \tau}
    \end{align*}
    where  we have defined $c_i = CV e_i$ (the $i$-th column of $CV$) and $b_i = (e_i^\top V^{-1}B)^\top $ (the $i$-th row of $V^{-1}B$). 
    We can write 
    \[S_i = c_i b_i^\top \in \mathbb{R}^{p\times m} \text{ for } i \in [d], \text{ and  }
        \tilde{u}_t = \begin{bmatrix}
            u_t &u_{t-1} & u_{t-2} & \cdots & u_1 & 0 & 0 & \cdots & 0 
        \end{bmatrix}
    \]
    where $\tilde{u}_t \in \mathbb{R}^{m\times T}$. Notice that we can bound the norm of $M_i$ as follows:
    \begin{align}
        \label{ineq:bound-M_i}
        \lVert S_i \rVert_{2} &= \lVert c_i b_i^\top \rVert_{2} = \lVert CV e_i e_i^\top V^{-1}B \rVert_{2} \leq \lVert CV e_i \rVert_{2} \cdot \lVert e_i^\top V^{-1}B \rVert_{2}
    \end{align}
    We can also bound the norm of $\tilde{u}_t$ as follows:
    \begin{align}
        \label{ineq:bound-tilde-u_t}
        \lVert \tilde{u}_t \rVert_2^2 &= \sum_{\tau=0}^{t-1} \lVert u_{t-\tau} \rVert_2^2 \leq \sum_{\tau=0}^{t-1} K^2 \leq K^2 \cdot  T
    \end{align}
    That means that the following quantity can be written as:
    \begin{align*}
        \sum_{\tau = 0}^{t-1} \lambda_i^{\tau} u_{t - \tau}
        &= \begin{bmatrix}
            u_t & u_{t-1} & u_{t-2} & \cdots & u_1 & 0 & 0 & \cdots & 0
        \end{bmatrix}
        \begin{bmatrix}
            1 \\
            \lambda_i \\
            \lambda_i^2 \\
            \vdots \\
            \lambda_i^{T-1}
        \end{bmatrix} = \tilde{u}_t \cdot \mu_{\lambda_i} \in \mathbb{R}^m
    \end{align*}
    We will now perform a change of basis:
    \begin{align*}
        y_t &= \sum_{i=1}^\hiddenDim S_i \left(\tilde{u}_t \cdot \mu_{\lambda_i}\right) = \sum_{i=1}^\hiddenDim S_i \tilde{u}_t \left( \sum_{\tau=1}^T \phi_\tau \phi_\tau^\top \right) \mu_{\lambda_i} \\
        &= \sum_{i=1}^\hiddenDim \sum_{\tau=1}^T \left(S_i \cdot \left(  \phi_\tau^\top \mu_{\lambda_i}  \right) \right) \tilde{u}_t \phi_\tau = \sum_{\tau=1}^T \sum_{i=1}^\hiddenDim M_{i,\tau} \tilde{u}_t \phi_\tau \\
        &= \sum_{\tau=1}^T M_{ \tau} \tilde{u}_t \phi_\tau
    \end{align*}
    where $M_{\tau} = \sum_{i=1}^\hiddenDim M_{i,\tau}$ and $M_{i,\tau} = S_i \cdot \left(  \phi_\tau^\top \mu_{\lambda_i}  \right)$ -- recall that $\phi_\tau^\top \mu_{\lambda_i}$ is just a scalar. We can bound the norm of $M_{\tau}$ which will be useful later:
    \begin{align}
        \label{ineq:bound-M_i_tau}
        \lVert M_{ \tau} \rVert_2 &= \left\lVert \sum_{i=1}^\hiddenDim M_{i,\tau} \right\rVert_2 = \left\lVert \sum_{i=1}^\hiddenDim S_i \cdot \left(  \phi_\tau^\top \mu_{\lambda_i}  \right) \right\rVert_2 \leq \sum_{i=1}^\hiddenDim \lVert S_i \rVert_2 \cdot |\phi_\tau^\top \mu_{\lambda_i}| \\
        &\overset{\text{Ineq \ref{ineq:bound-M_i}}}{\leq} \sum_{i=1}^\hiddenDim \lVert CV e_i \rVert_{2}  \cdot \lVert  e_i^\top V^{-1}B \rVert_{2} \cdot |\phi_\tau^\top \mu_{\lambda_i}| \\
        &\leq \lVert B \rVert_2 \cdot \lVert C \rVert_2  \sum_{i=1}^\hiddenDim \lVert V e_i \rVert_{2}  \cdot \lVert  e_i^\top V^{-1} \rVert_{2} \cdot |\phi_\tau^\top \mu_{\lambda_i}|  \\
        &\overset{\text{Thm \ref{thm:decay-eigenvectors}}}{\leq} 2\sqrt{\pi}\cdot \lVert B \rVert_2 \cdot \lVert C \rVert_2  \cdot \sum_{i=1}^\hiddenDim \lVert V e_i \rVert_{2}  \cdot \lVert  e_i^\top V^{-1} \rVert_{2} \cdot T^{\frac{1}{2}}e^{-\frac{\tau}{4\log{T}}} \\
        &\leq 2\sqrt{\pi}\cdot \lVert B \rVert_2 \cdot \lVert C \rVert_2  \cdot \lVert V \rVert_F \cdot \lVert V^{-1} \rVert_F \cdot T^{\frac{1}{2}}e^{-\frac{\tau}{4\log{T}}}
    \end{align}
    Notice that $y_t$ is the true output of the LDS. Consider a prediction $\hat{y}_t$ from the class $\mathrm{SF}_{h,R}$ which predicts:
    \[
        \hat{y}_t = \sum_{\tau=1}^{h}M_{ \tau} \left(\tilde{u}_t \cdot \phi_\tau \right)
    \]
    We have that:
    \begin{align*}
        \lVert \hat{y}_t - y_t\rVert_2 &= \left \lVert \sum_{\tau=1}^T M_{ \tau} \left(\tilde{u}_t \cdot \phi_\tau \right) - \sum_{\tau=1}^{h} M_{ \tau} \left(\tilde{u}_t \cdot \phi_\tau \right) \right \rVert_2 \\
        &= \left \lVert \sum_{\tau=h+1}^T M_{ \tau} \left(\tilde{u}_t \cdot \phi_\tau \right) \right \rVert_2 \\
        &\leq \sum_{\tau=h+1}^T \lVert M_{ \tau} \rVert_2 \cdot \lVert \tilde{u}_t \cdot \phi_\tau \rVert_2 \leq \sum_{\tau=h+1}^T \lVert M_{ \tau} \rVert_2 \cdot \lVert \tilde{u}_t\rVert_2 \cdot \lVert \phi_\tau \rVert_2 \displaybreak[1]\\
        &\overset{\text{Ineq. \ref{ineq:bound-tilde-u_t}}}{\leq} K\cdot \sqrt{T} \cdot \sqrt{T} \cdot \sum_{\tau=h+1}^T \lVert M_{ \tau} \rVert_2 \\
        &\overset{\text{Ineq. \ref{ineq:bound-M_i_tau}}}{\leq} T \cdot K \sum_{\tau=h+1}^T 2\sqrt{\pi}\cdot \lVert B \rVert_2 \cdot \lVert C \rVert_2  \cdot \lVert V \rVert_F \cdot \lVert V^{-1} \rVert_F \cdot T^{\frac{1}{2}}e^{-\frac{\tau}{4\log{T}}} \\
        &\leq 2\sqrt{\pi} \cdot T^{\frac{5}{2}} \cdot K \cdot \lVert B \rVert_F \cdot \lVert C \rVert_F\cdot \lVert V \rVert_F \cdot \lVert V^{-1} \rVert_F \cdot e^{-\frac{h}{4\log{T}}}
    \end{align*}
    Thus:
    \begin{align*}
        L &= \sum_{t=1}^T \lVert \hat{y}_t - y_t\rVert^2  \leq \sum_{t=1}^T\left( 2\sqrt{\pi} \cdot T^{\frac{5}{2}} \cdot K \cdot \lVert B \rVert_F\cdot \lVert V \rVert_F \cdot \lVert V^{-1} \rVert_F \cdot \lVert C \rVert_F\cdot e^{-\frac{h}{4\log{T}}} \right)^2 \\
        &= 4\pi \cdot K^2 \cdot \lVert C \rVert_{F}^2 \cdot \lVert B \rVert_{F}^2 \cdot \lVert V \rVert_F^2 \cdot \lVert V^{-1} \rVert_F^2 \cdot T^6 \cdot e^{-\frac{h}{2\log{T}}}
    \end{align*}
    We want $L \leq \varepsilon$ thus we need:
    \begin{align*}
        & 4\pi \cdot K^2 \cdot \lVert C \rVert_{F}^2 \cdot \lVert B \rVert_{F}^2 \cdot \lVert V \rVert_F^2 \cdot \lVert V^{-1} \rVert_F^2 \cdot T^6 \cdot e^{-\frac{h}{2\log{T}}} \leq \varepsilon \\
        & \iff e^{\frac{h}{2 \log{T}}} \geq \frac{4\pi \cdot K^2 \cdot \lVert C \rVert_{F}^2 \cdot \lVert B \rVert_{F}^2 \cdot \lVert V \rVert_F^2 \cdot \lVert V^{-1} \rVert_F^2 \cdot T^6}{\varepsilon} \\
        & \iff h \geq  2 \log{T} \cdot \log{\left( \frac{4\pi \cdot K^2 \cdot \lVert C \rVert_{F}^2 \cdot \lVert B \rVert_{F}^2 \cdot \lVert V \rVert_F^2 \cdot \lVert V^{-1} \rVert_F^2 \cdot T^6}{\varepsilon} \right)} 
    \end{align*}
    Using the constant $C_0 = 4\pi \cdot K^2 \cdot \lVert C \rVert_{F}^2 \cdot \lVert B \rVert_{F}^2 \cdot \lVert V \rVert_F^2 \cdot \lVert V^{-1} \rVert_F^2$ we get that we just need:
    \[
        h \geq 2 \cdot \log{T} \cdot \log{\left(\frac{C_0 \cdot T^6}{\varepsilon} \right)}
    \]
\end{proof}

\section{Diagonalization with separated real eigenvalues}
\label{app:diag}

\begin{lemma}\label{lem:diag}
  Let $A\in\mathbb{R}^{\hiddenDim\times \hiddenDim}$. Let
  $S\subseteq\mathbb{R}$ be a set of real eigenvalues, and assume
  that every eigenvalue in $S$ is semisimple. Let $n$ be the total algebraic
  multiplicity of the eigenvalues in $S$. Write
  \(
  \lambda_1,\dots,\lambda_n
  \)
  for the eigenvalues in $S$, counted with algebraic multiplicity, and write
  \(
  \mu_1,\dots,\mu_{\hiddenDim-n}
  \)
  for the remaining eigenvalues of $A$, counted with algebraic multiplicity.
  We denote by $\sigma(A)=\{\lambda_1,\dots,\lambda_n,\mu_1,\dots,\mu_{\hiddenDim-n}\}$
  the spectrum of $A$.
  
  Then there exist $P\in \mathrm{GL}_{\hiddenDim}(\mathbb{R})$ and
  $A_{\mathrm{ustb}}\in\mathbb{R}^{(\hiddenDim-n)\times(\hiddenDim-n)}$ such that
  \[
  A
  =
  P
  \begin{pmatrix}
  \operatorname{diag}(\lambda_1,\dots,\lambda_n) & 0\\
  0 & A_{\mathrm{ustb}}
  \end{pmatrix}
  P^{-1}.
  \]
  Moreover, the eigenvalues of $A_{\mathrm{ustb}}$ are precisely
  \(
  \mu_1,\dots,\mu_{\hiddenDim-n},
  \)
  counted with algebraic multiplicity.
  \end{lemma}
  
  \begin{proof}
  For each eigenvalue $\nu\in\sigma(A)$, let
  \[
  G_\nu:=\ker\bigl((A-\nu I)^{\hiddenDim}\bigr)
  \]
  denote the generalized eigenspace corresponding to $\nu$. By the primary
  decomposition theorem,
  \[
  \mathbb{C}^{\hiddenDim}
  =
  \bigoplus_{\nu\in\sigma(A)} G_\nu.
  \]
  Set
  \[
  U:=\bigoplus_{\rho\in S}G_\rho,
  \qquad
  W:=\bigoplus_{\nu\in\sigma(A)\setminus S}G_\nu.
  \]
  Thus $W$ is the generalized eigenspace corresponding to the eigenvalues not
  in $S$.
  Then
  \[
  \mathbb{C}^{\hiddenDim}=U\oplus W,
  \qquad
  \dim_{\mathbb C}U=n.
  \]
  
  Since every eigenvalue in $S$ is semisimple, we have
  \[
  G_\rho=\ker(A-\rho I)
  \qquad
  \text{for every }\rho\in S.
  \]
  Also, each $\rho\in S$ is real, and $\ker(A-\rho I)$ is the kernel of a real
  matrix. Hence it has a basis of real vectors. Choosing such bases and
  concatenating them, we get real eigenvectors
  \[
  q_1,\dots,q_n\in\mathbb{R}^{\hiddenDim}
  \]
  forming a basis of $U$, ordered so that
  \[
  Aq_i=\lambda_i q_i,
  \qquad
  i=1,\dots,n.
  \]
  
  Now $A$ is real, so for every eigenvalue $\nu$,
  \[
  \overline{G_\nu}=G_{\overline{\nu}}.
  \]
  Since $A$ is real, non-real eigenvalues appear together with their complex
  conjugates. Since $S\subseteq\mathbb{R}$ contains only real eigenvalues, if
  $\nu\notin S$ then also $\overline{\nu}\notin S$. Hence
  $\sigma(A)\setminus S$ is closed under complex conjugation, and by the
  identity above the subspace $W$ is stable under complex conjugation.
   Set
  \[
  U_{\mathbb R}:=\operatorname{span}_{\mathbb R}(q_1,\dots,q_n),
  \qquad
  W_{\mathbb R}:=W\cap\mathbb R^{\hiddenDim}.
  \]
  If $x\in\mathbb R^{\hiddenDim}$, write uniquely
  \[
  x=u+w,
  \qquad
  u\in U,\quad w\in W.
  \]
  Since
  \[
  x=\overline{x}=\overline{u}+\overline{w},
  \]
  with $\overline{u}\in U$ and $\overline{w}\in W$, uniqueness of the decomposition
  $\mathbb{C}^{\hiddenDim}=U\oplus W$ gives
  \[
  u=\overline{u},
  \qquad
  w=\overline{w}.
  \]
  Thus $u\in U_{\mathbb R}$ and $w\in W_{\mathbb R}$, so
  \[
  \mathbb R^{\hiddenDim}=U_{\mathbb R}\oplus W_{\mathbb R}.
  \]
  
  Choose a real basis
  \[
  w_1,\dots,w_{\hiddenDim-n}
  \]
  of $W_{\mathbb R}$, and define
  \[
  P
  :=
  \begin{bmatrix}
  q_1&\cdots&q_n&w_1&\cdots&w_{\hiddenDim-n}
  \end{bmatrix}
  \in \mathrm{GL}_{\hiddenDim}(\mathbb R).
  \]
  Both $U_{\mathbb R}$ and $W_{\mathbb R}$ are $A$-invariant. Hence, in this basis,
  \[
  P^{-1}AP
  =
  \begin{pmatrix}
  \operatorname{diag}(\lambda_1,\dots,\lambda_n) & 0\\
  0 & A_{\mathrm{ustb}}
  \end{pmatrix}
  \]
  for some real matrix
  \[
  A_{\mathrm{ustb}}\in\mathbb R^{(\hiddenDim-n)\times(\hiddenDim-n)}.
  \]
  Equivalently,
  \[
  A
  =
  P
  \begin{pmatrix}
  \operatorname{diag}(\lambda_1,\dots,\lambda_n) & 0\\
  0 & A_{\mathrm{ustb}}
  \end{pmatrix}
  P^{-1}.
  \]
  
  Finally, $A_{\mathrm{ustb}}$ represents the restriction of $A$ to the invariant
  subspace $W$. Therefore its eigenvalues are exactly the remaining eigenvalues
  \[
  \mu_1,\dots,\mu_{\hiddenDim-n},
  \]
  counted with algebraic multiplicity. This proves the result.
  \end{proof}

\section{Autoregressive representation of any LDS}
\label{app:ar-only}

\begin{proposition}\label{lem:ar-only}
Let $\{u_t\}_{t=1}^T \subset \mathbb{R}^{\inputDim}$ be any input sequence, and consider the noiseless LDS
\[
x_t = A x_{t-1} + B u_t,
\qquad
y_t = C x_t,
\]
where
\[
A \in \mathbb{R}^{\hiddenDim\times \hiddenDim},
\qquad
B \in \mathbb{R}^{\hiddenDim\times \inputDim},
\qquad
C \in \mathbb{R}^{\outputDim\times \hiddenDim}.
\]
Then there exist coefficients
\[
\{\alpha_i\}_{i=1}^{\hiddenDim} \subset \mathbb{R},
\qquad
\{M_i\}_{i=0}^{\hiddenDim-1} \subset \mathbb{R}^{\outputDim\times \inputDim},
\]
depending only on $(A,B,C)$, such that for every $t=\hiddenDim+1,\dots,T$,
\[
\left\|
y_t - \sum_{i=1}^{\hiddenDim} \alpha_i\, y_{t-i}
      - \sum_{i=0}^{\hiddenDim-1} M_i\, u_{t-i}
\right\|_2 = 0.
\]

Moreover, if $\lambda_1,\dots,\lambda_{\hiddenDim}$ are the eigenvalues of $A$ (counted with algebraic multiplicity), then the coefficients may be chosen so that, for all relevant $i$,
\[
|\alpha_i| \le L,
\qquad
\|M_i\|_\infty \le L,
\]
where
\[
L
:=
\left(\prod_{j=1}^{\hiddenDim} (1+|\lambda_j|)\right)
\left(1 + \hiddenDim \|C\|_\infty \|B\|_\infty \max\{1,\|A\|_\infty\}^{\,\hiddenDim-1}\right).
\]
\end{proposition}

\begin{proof}
Let
\[
\chi_A(z)=z^{\hiddenDim}+c_1 z^{\hiddenDim-1}+\cdots+c_{\hiddenDim}=\prod_{j=1}^{\hiddenDim} (z-\lambda_j)
\]
be the characteristic polynomial of $A$, and set $c_0:=1$. By Cayley--Hamilton,
\[
A^{\hiddenDim} + c_1 A^{\hiddenDim-1} + \cdots + c_{\hiddenDim} I = 0.
\]

For any $s\ge 1$, repeated substitution in the state recursion gives
\[
x_t = A^s x_{t-s} + \sum_{j=0}^{s-1} A^j B u_{t-j}.
\]
Applying this with $s=\hiddenDim$ yields
\[
x_t = A^{\hiddenDim} x_{t-\hiddenDim} + \sum_{j=0}^{\hiddenDim-1} A^j B u_{t-j},
\]
and, for each $i=1,\dots,\hiddenDim$,
\[
x_{t-i} = A^{\hiddenDim-i} x_{t-\hiddenDim} + \sum_{j=0}^{\hiddenDim-i-1} A^j B u_{t-i-j}.
\]
Multiplying the latter identity by $c_i$, summing over $i=1,\dots,\hiddenDim$, and adding the first identity, we obtain
\[
x_t + \sum_{i=1}^{\hiddenDim} c_i x_{t-i}
=
\Bigl(A^{\hiddenDim}+\sum_{i=1}^{\hiddenDim} c_i A^{\hiddenDim-i}\Bigr)x_{t-\hiddenDim}
+
\sum_{j=0}^{\hiddenDim-1}\Bigl(\sum_{\ell=0}^j c_\ell A^{\,j-\ell}\Bigr) B u_{t-j}.
\]
The term multiplying $x_{t-\hiddenDim}$ vanishes by Cayley--Hamilton, hence
\[
x_t + \sum_{i=1}^{\hiddenDim} c_i x_{t-i}
=
\sum_{j=0}^{\hiddenDim-1}\Bigl(\sum_{\ell=0}^j c_\ell A^{\,j-\ell}\Bigr) B u_{t-j}.
\]
Multiplying by $C$ and using $y_t=Cx_t$, we get
\[
y_t + \sum_{i=1}^{\hiddenDim} c_i y_{t-i}
=
\sum_{j=0}^{\hiddenDim-1} C\Bigl(\sum_{\ell=0}^j c_\ell A^{\,j-\ell}\Bigr) B u_{t-j}.
\]
Therefore the claimed identity holds with
\[
\alpha_i := -c_i, \qquad i=1,\dots,\hiddenDim,
\]
and
\[
M_j := C\Bigl(\sum_{\ell=0}^j c_\ell A^{\,j-\ell}\Bigr) B,
\qquad j=0,\dots,\hiddenDim-1.
\]

It remains to prove the bounds. Set
\[
K:=\prod_{j=1}^{\hiddenDim} (1+|\lambda_j|),
\qquad
\beta:=\max\{1,\|A\|_\infty\}.
\]
Since $c_i = (-1)^i e_i(\lambda_1,\dots,\lambda_{\hiddenDim})$, where $e_i$ is the $i$th elementary symmetric polynomial,
\[
|c_i|
\le e_i(|\lambda_1|,\dots,|\lambda_{\hiddenDim}|)
\le \prod_{j=1}^{\hiddenDim} (1+|\lambda_j|)
=K.
\]
Hence
\[
|\alpha_i| = |c_i| \le K.
\]
Also, for each $j=0,\dots,\hiddenDim-1$,
\[
    \|M_j\|_\infty
    \le
    \|C\|_\infty \|B\|_\infty
    \sum_{\ell=0}^j |c_\ell|\, \|A\|_\infty^{\,j-\ell}
    \le
    \|C\|_\infty \|B\|_\infty
    \sum_{\ell=0}^j K\, \beta^{\,j-\ell}
    \le
    \hiddenDim K \|C\|_\infty \|B\|_\infty \beta^{\,\hiddenDim-1}.
    \]
Thus both families of coefficients are bounded by
\[
L
:=
K\Bigl(1+\hiddenDim\|C\|_\infty\|B\|_\infty\beta^{\,\hiddenDim-1}\Bigr)
=
\left(\prod_{j=1}^{\hiddenDim} (1+|\lambda_j|)\right)
\left(1+\hiddenDim\|C\|_\infty\|B\|_\infty \max\{1,\|A\|_\infty\}^{\,\hiddenDim-1}\right).
\]
This proves the result.
\end{proof}

\section{Spectral filtering approximation: stable real-diagonalizable case}
\label{app:sf-diagonal-real}

\begin{proposition}
\label{lem:sf-diagonal-real}
Let $\{u_t\}_{t=1}^T \subset \mathbb{R}^{\inputDim}$ be as in \cref{sec:setup}, and consider the noiseless LDS
\[
x_t = A x_{t-1} + B u_t,
\qquad
y_t = C x_t,
\qquad
x_0 = 0,
\]
where $A = VDV^{-1}$ for some $V\in \mathrm{GL}_{\hiddenDim}(\mathbb R)$ and
$D=\operatorname{diag}(\lambda_1,\dots,\lambda_{\hiddenDim})$ with
$\lambda_1,\dots,\lambda_{\hiddenDim} \in (-1,1)$, $B \in \mathbb{R}^{\hiddenDim\times \inputDim}$, and $C \in \mathbb{R}^{\outputDim\times \hiddenDim}$. Let $\phi_1,\dots,\phi_T \in \mathbb{R}^T$ denote the spectral filters as defined in \cref{def:spectral-filters}, and for each $t=1,\dots,T$, let $\tilde u_t \in \mathbb{R}^{\inputDim\times T}$ denote the zero-padded input history.

Fix $\varepsilon > 0$ and set the number of filters to
\[
h \;=\; h(\varepsilon,T) \;:=\; 2\,\log T \,\cdot\, \log\!\left(\frac{C_0\, T^6}{\varepsilon}\right),
\]
where $C_0 = 4\pi \cdot K^2 \cdot \lVert C \rVert_{F}^2 \cdot \lVert B \rVert_{F}^2 \cdot \lVert V \rVert_F^2 \cdot \lVert V^{-1} \rVert_F^2 $.
 Then there exist coefficients
\[
\{M_{s} \mid s \in [h] \} \;\subset\; \mathbb{R}^{\outputDim\times \inputDim},
\]
depending only on $(A,B,C)$, such that the predictor
\[
\pred_t \;:=\; \sum_{s=1}^h  M_{s}\,\bigl(\tilde u_t\,\phi_s\bigr)
\]

satisfies
\[
\sum_{t=1}^T \bigl\|\,y_t - \pred_t\,\bigr\|_2^{\,2}
\;\le\; \varepsilon.
\]

Moreover, the coefficients may be chosen so that
\[
 \|M_{s}\|_2
\;\le\;
\lVert V \rVert_F \cdot \lVert V^{-1} \rVert_F \cdot \lVert C \rVert_{F}  \cdot \lVert B \rVert_{F}  \cdot \min\left(\frac{1}{1- \max_{1\le i\le \hiddenDim} |\lambda_i|} ,2 \sqrt{\pi} \cdot \sqrt{T}\cdot e^{-\frac{s}{4\log{T}}}\right)
\]
where $V$ and $D$ are the real diagonalization above.
\end{proposition}

\begin{proof}
  This lemma is identical to \cref{thm:approx-spectral-filtering} but we still have to prove the bounds on the norm of $M_{s}$. 
  Following the notation we used in the proof of \cref{thm:approx-spectral-filtering}, we have that:
  \begin{align*}
    \lVert M_s \rVert_2 \le \sum_{r=1}^{\hiddenDim} \|M_{r,s}\|_2 &= \sum_{r=1}^{\hiddenDim} \left\| S_r \left(  \phi_s^\top \mu_{\lambda_r}  \right) \right\|_2 = \sum_{r=1}^{\hiddenDim} \left\| S_r \right\|_2 \cdot |\phi_s^\top \mu_{\lambda_r} | \\
  \end{align*}
  Notice that:
  \begin{align*}
    \sum_{r=1}^{\hiddenDim} \left\| S_r \right\|_2
    &= \sum_{r=1}^{\hiddenDim} \lVert CVe_i (e_i^\top V^{-1}B) \rVert_2 \\
    &\leq \sum_{r=1}^{\hiddenDim} \lVert CVe_i \rVert_2 \cdot \lVert e_i^\top V^{-1}B \rVert_2 \\
    &\leq \lVert V \rVert_F \cdot \lVert V^{-1} \rVert_F \cdot\lVert C \rVert_F \cdot \lVert B \rVert_F
  \end{align*}
  Also, we have that:
  \[
    |\phi_s^\top \mu_{\lambda_r}| \leq \lVert \phi_s \rVert_2 \cdot \lVert \mu_{\lambda_r} \rVert_2 \leq 1 \cdot \frac{1}{1 - |\lambda_r|}
  \]
  since $\phi_s$ has norm $1$ and the entries of $\mu_{\lambda_r}$ sum up to be at most $\frac{1}{1 - |\lambda_r|}$. Combining this with
  \cref{thm:decay-eigenvectors} we get:
  \[
    |\phi_s^\top \mu_{\lambda_r}| \leq 2\sqrt{\pi} \cdot T^{\frac{1}{2}} \cdot  e^{-\frac{s}{4\log{T}}}
  \]
  Combining the above inequalities we get the desired bound.
\end{proof}
\section{Finite-memory filters for strictly stable LDS}
\label{app:fast-fir}

The following lemma formalises \cref{sec:fir-approx}: a strictly stable
semisimple LDS admits a finite-memory input filter of length
independent of the state dimension.

\begin{lemma}\label{lem:fast-fir}
Consider an LDS $x_t = A x_{t-1}+B u_t$, $y_t = C x_t$ with $x_0=0$.
Suppose $A$ is semisimple over $\mathbb{C}$ and every eigenvalue satisfies
$|\lambda|\le 1-1/L$ for some $L\ge 1$. Write $A=VDV^{-1}$ and set
$C_{\mathrm{fir},0}:=\|CV\|\,\|V^{-1}B\|$. Then, for every $\varepsilon>0$, setting
\[
  L_\varepsilon
  \;=\;
  \left\lceil
  L\,\log\!\left(\frac{K\,C_{\mathrm{fir},0}\,L\sqrt{T}}{\sqrt{\varepsilon}}\right)
  \right\rceil,
\]
implies that there exists matrices $Q_0,\dots,Q_{L_\varepsilon-1}$ such that
\[
  \sum_{\ell=0}^{L_\varepsilon-1}\|Q_\ell\|\le C_{\mathrm{fir},0}L,
\]
and
\[
  \sum_{t=1}^{T}\Bigl\|y_t-\sum_{\ell=0}^{L_\varepsilon-1}Q_\ell\,u_{t-\ell}\Bigr\|_2^{2}
  \;\le\; \varepsilon.
\]
\end{lemma}

\begin{proof}
Since $D$ is diagonal and $\rho(D)\le 1-1/L$,
\[
  \|C A^\tau B\|
  \le
  \|CV\|\,\|V^{-1}B\|\,(1-1/L)^{\tau},
\]
Set $Q_\ell:=C A^{\ell} B$ for
$\ell=0,\dots,L_\varepsilon-1$ and recall that
$y_t=\sum_{\ell\ge 0} C A^\ell B\,u_{t-\ell}$
is the exact convolution. The truncation error at time $t$ is at most
\[
  \sum_{\tau\ge L_\varepsilon}\|C A^\tau B\|\,K
  \;\le\;
  K\,C_{\mathrm{fir},0}\,L\,(1-1/L)^{L_\varepsilon}
  \;\le\;
  K\,C_{\mathrm{fir},0}\,L\,e^{-L_\varepsilon/L}.
\]
Choosing $L_\varepsilon$ as in the statement gives
\[
  e^{-L_\varepsilon/L}
  \le
  \frac{\sqrt{\varepsilon}}{K\,C_{\mathrm{fir},0}\,L\sqrt{T}}.
\]
Thus, for every $t\le T$,
\[
  \Bigl\|y_t-\sum_{\ell=0}^{L_\varepsilon-1}Q_\ell u_{t-\ell}\Bigr\|_2
  \le \sqrt{\varepsilon/T},
\]
and summing over $t=1,\dots,T$ gives the claimed total error bound.
Finally,
\[
  \sum_{\ell=0}^{L_\varepsilon-1}\|Q_\ell\|
  \le
  C_{\mathrm{fir},0}\sum_{\ell=0}^{L_\varepsilon-1}(1-1/L)^\ell
  \le C_{\mathrm{fir},0}L,
\]
which proves the lemma.
\end{proof}

\section{Unified approximation: AR plus spectral and finite-memory filters}
\label{app:ar-plus-sf}

\begin{theorem}
  \label{lem:ar-plus-sf}
  Consider the following LDS
  \[
  x_t = A x_{t-1} + B u_t,
  \qquad
  y_t = C x_t,
  \qquad
  x_0 = 0,
  \]
  where $A\in\mathbb{R}^{\hiddenDim\times \hiddenDim}$, $B\in\mathbb{R}^{\hiddenDim\times\inputDim}$,
  and $C\in\mathbb{R}^{\outputDim\times \hiddenDim}$. Extend the input and output sequences by $u_t:=0$ and
  $y_t:=0$ for all $t\le 0$, and define $\tilde u_t$ using this zero-padded
  input sequence.

  Define the split parameter
  \[
  \taustb:=\log(\ustb+2).
  \]
  Let $\mu_1,\dots,\mu_q$ be the eigenvalues of $A$ counted by the instability
  complexity~\cref{def:ustb}---the non-semisimple eigenvalues together with the
  semisimple eigenvalues outside $(-1,1)$ with $|\lambda_i|>1-1/\taustb$, so that
  $q\le\ustb$; if $q<\ustb$ we pad the AR coefficients below by zeros. Let
  $\lambda_1,\dots,\lambda_r$ be the real semisimple eigenvalues of $A$ with
  \[
  1-\frac{1}{\taustb} \;<\; |\lambda_a| \;<\; 1,
  \qquad a=1,\dots,r,
  \]
  counted with algebraic multiplicity; we call these the \emph{spectral
  group}. The remaining $\hiddenDim-r-q$ eigenvalues are semisimple with
  modulus at most $1-1/\taustb$ and form the \emph{strictly stable} block.

  Using the real invariant subspaces associated with these three spectral
  groups (\cref{lem:diag} applied separately to the spectral and to the
  strictly stable parts), there exist $P\in \mathrm{GL}_{\hiddenDim}(\mathbb R)$ and
  matrices $A_{\mathrm{ustb}}\in\mathbb R^{q\times q}$ and
  $A_{\mathrm{fast}}\in\mathbb R^{(\hiddenDim-r-q)\times(\hiddenDim-r-q)}$
  (with $A_{\mathrm{fast}}$ semisimple of spectral radius at most
  $1-1/\taustb$) such that
  \[
  A
  =
  P
  \begin{pmatrix}
  \Lambda & 0 & 0\\
  0 & A_{\mathrm{ustb}} & 0\\
  0 & 0 & A_{\mathrm{fast}}
  \end{pmatrix}
  P^{-1},
  \qquad
  \Lambda:=\operatorname{diag}(\lambda_1,\dots,\lambda_r).
  \]
  The eigenvalues of $A_{\mathrm{ustb}}$ are precisely $\mu_1,\dots,\mu_q$
  (counted with algebraic multiplicity).

  Define the real block matrices
  \[
  P^{-1}B
  =
  \begin{bmatrix}
  B_{\mathrm{st}}\\
  B_{\mathrm{ustb}}\\
  B_{\mathrm{fast}}
  \end{bmatrix},
  \qquad
  CP
  =
  \begin{bmatrix}
  C_{\mathrm{st}} & C_{\mathrm{ustb}} & C_{\mathrm{fast}}
  \end{bmatrix},
  \]
  where $B_{\mathrm{st}}\in\mathbb R^{r\times\inputDim}$,
  $B_{\mathrm{ustb}}\in\mathbb R^{q\times\inputDim}$,
  $C_{\mathrm{st}}\in\mathbb R^{\outputDim\times r}$, and
  $C_{\mathrm{ustb}}\in\mathbb R^{\outputDim\times q}$.

  Define
  \[
  L_{\mathrm{ustb}}
  :=
  \left(\prod_{j=1}^{q} (1+|\mu_j|)\right)
  \left(
  1+\ustb\,\|C_{\mathrm{ustb}}\|_\infty\,\|B_{\mathrm{ustb}}\|_\infty\,
  \max\{1,\|A_{\mathrm{ustb}}\|_\infty\}^{\,\max\{q-1,0\}}
  \right)
  \]
  and
  \[
  \widetilde L_{\mathrm{ustb}}
  :=
  (1+\ustb L_{\mathrm{ustb}})\left(1-\frac{1}{\taustb}\right)^{-\ustb},
  \]
  so that $\log\widetilde L_{\mathrm{ustb}}
  =O(\log(1+\ustb L_{\mathrm{ustb}})+\ustb/\log(\ustb+2))$ after absorbing the
  finitely many small values of $\ustb$ into the constant.
  Also set
  \[
  \rho_{\mathrm{st}}
  :=
  \max_{1\le a\le r}|\lambda_a|,
  \]
  with the convention $\rho_{\mathrm{st}}=0$ if $r=0$. Define
  $C_{\mathrm{fast},0}:=0$ if the fast block is empty. Otherwise, write
  $A_{\mathrm{fast}}=V_{\mathrm{fast}}D_{\mathrm{fast}}V_{\mathrm{fast}}^{-1}$
  over $\mathbb C$ and set
  \[
  C_{\mathrm{fast},0}
  :=
  \|C_{\mathrm{fast}}V_{\mathrm{fast}}\|\,
  \|V_{\mathrm{fast}}^{-1}B_{\mathrm{fast}}\|.
  \]
  Define
  \[
  h
  :=
  \left\lceil
  2\log T \cdot
  \log\!\left(
  \frac{
  4\pi K^2 \,\widetilde L_{\mathrm{ustb}}^{2}\,\|C_{\mathrm{st}}\|_{F}^{2}\|B_{\mathrm{st}}\|_{F}^{2}
  \,T^6\,(1+\ustb L_{\mathrm{ustb}})^2
  }{\varepsilon}
  \right)
  \right\rceil
  \]
  and
  \[
  L
  :=
  \left\lceil
  \taustb\,\log\!\left(
  \frac{(1+C_{\mathrm{fast},0})\,K\,(1+\taustb)\,T\,(1+\ustb L_{\mathrm{ustb}})^2}{\varepsilon}
  \right)
  \right\rceil .
  \]
  For every $s=1,\dots,h$, define
  \[
  R_s
  :=
  \|C_{\mathrm{st}}\|_{F}\,\|B_{\mathrm{st}}\|_{F}
  \cdot
  \min\!\left(
  \frac{1}{1-\rho_{\mathrm{st}}},
  2\sqrt{\pi}\sqrt{T}\,e^{-s/(4\log T)}
  \right).
  \]

  Then there exist coefficients
  \[
  \{\alpha_i\}_{i=1}^{\ustb} \subset \mathbb{R},
  \qquad
  \{W_{s}\}_{1\le s\le h}
  \subset
  \mathbb{R}^{\outputDim\times\inputDim},
  \qquad
  \{Q_\ell\}_{0\le \ell\le L+\ustb} \subset \mathbb{R}^{\outputDim\times\inputDim},
  \]
  such that, for every $t=1,\dots,T$, the predictor
  \[
  \pred_t
  =
  \sum_{i=1}^{\ustb} \alpha_i\, y_{t-i}
  +
  \sum_{s=1}^{h}
  W_{s}\,(\tilde u_{t}\phi_s)
  +
  \sum_{\ell=0}^{L+\ustb} Q_\ell\,u_{t-\ell}
  \]
  satisfies
  \[
  \sum_{t=1}^{T} \|y_t-\pred_t\|_2^2 \le \varepsilon .
  \]

  Moreover, for all relevant indices,
  \[
  |\alpha_i|\le L_{\mathrm{ustb}},
  \qquad
  \|W_{s}\|_F \le \widetilde L_{\mathrm{ustb}}\, R_s,
  \]
  and the $Q_\ell$ satisfy
  \[
  \|Q_\ell\|_2\le
  L_Q
  :=
  \inputDim L_{\mathrm{ustb}}
  +
  \widetilde L_{\mathrm{ustb}}\,\|C_{\mathrm{st}}\|_F\,\|B_{\mathrm{st}}\|_F
  +
  (1+\ustb L_{\mathrm{ustb}})\,C_{\mathrm{fast},0},
  \]
  a constant depending only on the fixed system parameters.
\end{theorem}

\begin{proof}
Write $z_t := P^{-1}x_t = (s_t,e_t,f_t)$ according to the spectral, AR, and
strictly stable blocks. The transformed dynamics are
\[
s_t = \Lambda s_{t-1} + B_{\mathrm{st}} u_t,\qquad
e_t = A_{\mathrm{ustb}} e_{t-1} + B_{\mathrm{ustb}} u_t,\qquad
f_t = A_{\mathrm{fast}} f_{t-1} + B_{\mathrm{fast}} u_t,
\]
and
\[
y_t = C_{\mathrm{st}} s_t + C_{\mathrm{ustb}} e_t + C_{\mathrm{fast}} f_t
=: y_t^{\mathrm{st}} + y_t^{\mathrm{ustb}} + y_t^{\mathrm{fast}}.
\]
We may assume $\ustb\ge 1$; otherwise the AR component is empty and the proof
collapses to a direct application of \cref{lem:sf-diagonal-real} and
\cref{lem:fast-fir}.

Set
\[
\eta:=\frac{\varepsilon}{4(1+\ustb L_{\mathrm{ustb}})^2}.
\]

\paragraph{AR component via Cayley--Hamilton.}
Apply Lemma~\ref{lem:ar-only} to the $q$-dimensional subsystem
$e_t=A_{\mathrm{ustb}}e_{t-1}+B_{\mathrm{ustb}}u_t$,
$y_t^{\mathrm{ustb}}=C_{\mathrm{ustb}}e_t$, then pad the resulting AR
coefficients by zeros so that they have length $\ustb$. This yields
\[
\{\alpha_i\}_{i=1}^{\ustb}\subset\mathbb R,
\qquad
\{N_i\}_{i=0}^{\ustb-1}\subset\mathbb R^{\outputDim\times\inputDim},
\]
with $|\alpha_i|\le L_{\mathrm{ustb}}$ and $\|N_i\|_\infty\le L_{\mathrm{ustb}}$,
satisfying (after extending $u_t,y_t^{\mathrm{ustb}}$ by zero for $t\le 0$)
\[
y_t^{\mathrm{ustb}}
=
\sum_{i=1}^{\ustb}\alpha_i y_{t-i}^{\mathrm{ustb}}
+
\sum_{i=0}^{\ustb-1}N_i u_{t-i},
\qquad t=1,\dots,T.
\]
In particular $\sum_{i=1}^{\ustb}|\alpha_i|\le \ustb L_{\mathrm{ustb}}$. We
perform this step first because the coefficients $\alpha_i$ are used in the
treatment of the spectral block below. Write
\[
p_\alpha(z) := z^{\ustb}-\sum_{j=1}^{\ustb}\alpha_j z^{\ustb-j}
\]
for the associated monic polynomial, so that
$|p_\alpha(z)|\le 1+\ustb L_{\mathrm{ustb}}$ whenever $|z|\le 1$.

\paragraph{A shift identity for a single decay mode.}
The goal of this paragraph and the next is to express the AR residual
$y_t^{\mathrm{st}}-\sum_{i=1}^{\ustb}\alpha_i y_{t-i}^{\mathrm{st}}$ of the
spectral component using only features available at time $t$ (the current
input history), with no reference to the lagged histories at times
$t-1,\dots,t-\ustb$. We first prove the required identity for a single scalar
decay mode; the next paragraph applies it to each eigenvalue of $\Lambda$ and
sums.

For $\lambda\in(-1,1)\setminus\{0\}$, define the geometric convolution of the
input sequence with the decaying weights $1,\lambda,\lambda^2,\dots$:
\[
G_\lambda(t) := \sum_{\tau=0}^{t-1}\lambda^\tau\, u_{t-\tau}\in\mathbb R^{\inputDim},
\qquad t\ge 1,
\]
and $G_\lambda(t):=0$ for $t\le 0$. This is the convention-consistent
extension: for $t\le 0$ every input appearing in the sum is zero by the
zero-padding $u_s=0$, $s\le 0$. Splitting off the $\tau=0$ term and
re-indexing the remaining sum shows that $G_\lambda$ obeys the one-step
recurrence
\[
G_\lambda(t)
= u_t + \sum_{\tau=1}^{t-1}\lambda^{\tau}u_{t-\tau}
= u_t + \lambda\sum_{\sigma=0}^{t-2}\lambda^{\sigma}u_{(t-1)-\sigma}
= u_t + \lambda\, G_\lambda(t-1),
\qquad t\ge 1 .
\]
Solving this recurrence for the lagged value gives
$G_\lambda(t-1)=\lambda^{-1}\bigl(G_\lambda(t)-u_t\bigr)$: the value of the
convolution one step in the past is recovered from its current value by
dividing out one factor of $\lambda$ and removing the contribution of the most
recent input. Iterating this observation expresses \emph{any} lag of
$G_\lambda$ through its current value:

\begin{claim}
\label{clm:shift}
For every $j\ge 0$ and every $t\in\mathbb Z$,
\begin{equation}
\label{eq:shift}
G_\lambda(t-j)
=
\lambda^{-j}G_\lambda(t)
-
\sum_{i=0}^{j-1}\lambda^{i-j}\,u_{t-i}.
\end{equation}
\end{claim}

\begin{claimproof}
Induction on $j$. For $j=0$ both sides equal $G_\lambda(t)$ (the empty sum is
zero). Assume \eqref{eq:shift} holds for some $j\ge 0$. Applying the relation
$G_\lambda(s-1)=\lambda^{-1}\bigl(G_\lambda(s)-u_s\bigr)$ at $s=t-j$ and then
substituting the inductive hypothesis for $G_\lambda(t-j)$,
\[
G_\lambda(t-j-1)
=
\lambda^{-1}G_\lambda(t-j)-\lambda^{-1}u_{t-j}
=
\lambda^{-(j+1)}G_\lambda(t)
-\sum_{i=0}^{j-1}\lambda^{i-(j+1)}u_{t-i}
-\lambda^{-1}u_{t-j}.
\]
The last term is the $i=j$ summand $\lambda^{i-(j+1)}u_{t-i}\big|_{i=j}$, so
the two negative terms merge into $\sum_{i=0}^{j}\lambda^{i-(j+1)}u_{t-i}$,
which is \eqref{eq:shift} for $j+1$. The zero-padding conventions make the
recurrence, and hence every step above, valid for all $t\in\mathbb Z$,
including $t-j\le 0$; no boundary cases arise.
\end{claimproof}

In words, \eqref{eq:shift} says: a lagged convolution $G_\lambda(t-j)$ equals
the \emph{current} convolution $G_\lambda(t)$ rescaled by $\lambda^{-j}$,
minus a correction involving only the $j$ most recent inputs
$u_t,\dots,u_{t-j+1}$ (which the current convolution contains but the lagged
one does not).

We now apply \eqref{eq:shift} to the AR residual of $G_\lambda$. Substituting
\eqref{eq:shift} for each lag $j=1,\dots,\ustb$,
\[
G_\lambda(t)-\sum_{j=1}^{\ustb}\alpha_j\, G_\lambda(t-j)
=
G_\lambda(t)
-\sum_{j=1}^{\ustb}\alpha_j\lambda^{-j}\,G_\lambda(t)
+\sum_{j=1}^{\ustb}\alpha_j\sum_{i=0}^{j-1}\lambda^{i-j}\,u_{t-i}.
\]
Two collections of terms appear. First, the multiplier of $G_\lambda(t)$ is
\[
1-\sum_{j=1}^{\ustb}\alpha_j\lambda^{-j}
=
\frac{\lambda^{\ustb}-\sum_{j=1}^{\ustb}\alpha_j\lambda^{\ustb-j}}{\lambda^{\ustb}}
=
\frac{p_\alpha(\lambda)}{\lambda^{\ustb}},
\]
by the definition of $p_\alpha$. Second, in the double sum, exchanging the
order of summation (the pairs $(j,i)$ with $1\le j\le\ustb$, $0\le i\le j-1$
are exactly the pairs with $0\le i\le \ustb-1$, $i+1\le j\le\ustb$) collects
the coefficient of each fixed input $u_{t-i}$:
\[
\sum_{j=1}^{\ustb}\alpha_j\sum_{i=0}^{j-1}\lambda^{i-j}\,u_{t-i}
=
\sum_{i=0}^{\ustb-1}
\underbrace{\Bigl(\sum_{j=i+1}^{\ustb}\alpha_j\,\lambda^{\,i-j}\Bigr)}_{=:\;r_i(\lambda)}
\,u_{t-i}.
\]
Altogether we have shown, for every $t$,
\begin{equation}
\label{eq:residual-mode}
G_\lambda(t)-\sum_{j=1}^{\ustb}\alpha_j G_\lambda(t-j)
=
\frac{p_\alpha(\lambda)}{\lambda^{\ustb}}\,G_\lambda(t)
+
\sum_{i=0}^{\ustb-1}r_i(\lambda)\,u_{t-i}.
\end{equation}
The left-hand side involves the convolution at $\ustb+1$ different times
$t,t-1,\dots,t-\ustb$; the right-hand side involves it only at time $t$, plus
the $\ustb$ most recent inputs. This is the mechanism that will remove the
shifted spectral-filter features.

The rescaling factor and the input coefficients are controlled as follows.
Since $|\lambda|<1$, each power satisfies $|\lambda|^{\ustb-j}\le 1$, so
$|p_\alpha(\lambda)|\le 1+\sum_{j=1}^{\ustb}|\alpha_j|\le 1+\ustb L_{\mathrm{ustb}}$,
using the bound $|\alpha_j|\le L_{\mathrm{ustb}}$ from the AR step. If in
addition $|\lambda|>1-1/\taustb$, then
$|\lambda|^{-\ustb}\le\left(1-\frac{1}{\taustb}\right)^{-\ustb}$, hence
\begin{equation}
\label{eq:rescale-bound}
\begin{aligned}
\Bigl|\frac{p_\alpha(\lambda)}{\lambda^{\ustb}}\Bigr|
&\le
(1+\ustb L_{\mathrm{ustb}})\left(1-\frac{1}{\taustb}\right)^{-\ustb}
=
\widetilde L_{\mathrm{ustb}},\\
|r_i(\lambda)|
&\le
\sum_{j=i+1}^{\ustb}|\alpha_j|\,|\lambda|^{\,i-j}
\le
\Bigl(\sum_{j=1}^{\ustb}|\alpha_j|\Bigr)|\lambda|^{-\ustb}
\le
\widetilde L_{\mathrm{ustb}},
\end{aligned}
\end{equation}
where in the bound on $r_i(\lambda)$ we used $|\lambda|^{\,i-j}\le
|\lambda|^{-\ustb}$, valid because $|\lambda|<1$ and
$0< j-i\le \ustb$. It is exactly here that the spectral group's lower
threshold $1-1/\taustb$ is used: without it, the factor
$|\lambda|^{-\ustb}$ would be unbounded as $\lambda\to 0$.

\paragraph{Spectral component: AR residual first, then spectral filters.}
We first express $y^{\mathrm{st}}$ in the form required by the previous
paragraph, namely as a weighted sum of single-mode convolutions. Unrolling the
recursion $s_t=\Lambda s_{t-1}+B_{\mathrm{st}}u_t$ from $s_0=0$ gives
$s_t=\sum_{\tau=0}^{t-1}\Lambda^{\tau}B_{\mathrm{st}}\,u_{t-\tau}$, and hence
\[
y_t^{\mathrm{st}}
=
C_{\mathrm{st}}\,s_t
=
\sum_{\tau=0}^{t-1} C_{\mathrm{st}}\Lambda^\tau B_{\mathrm{st}}\,u_{t-\tau}.
\]
Because $\Lambda=\operatorname{diag}(\lambda_1,\dots,\lambda_r)$ is diagonal,
its powers decompose as
$\Lambda^{\tau}=\sum_{a=1}^{r}\lambda_a^{\tau}\,e_ae_a^{\!\top}$, where
$e_1,\dots,e_r$ is the standard basis of $\mathbb R^{r}$. Therefore
\[
C_{\mathrm{st}}\Lambda^{\tau}B_{\mathrm{st}}
=
\sum_{a=1}^{r}\lambda_a^{\tau}\,
\underbrace{(C_{\mathrm{st}}e_a)}_{a\text{-th column of }C_{\mathrm{st}}}
\underbrace{(e_a^{\!\top}B_{\mathrm{st}})}_{a\text{-th row of }B_{\mathrm{st}}}
=
\sum_{a=1}^{r}\lambda_a^{\tau}\,c_a,
\qquad
c_a := (C_{\mathrm{st}})_{:,a}\,(B_{\mathrm{st}})_{a,:}
\in\mathbb R^{\outputDim\times\inputDim},
\]
where each $c_a$ is a fixed rank-one matrix recording how the input enters
mode $a$ and how mode $a$ appears in the output. Substituting this into the
expression for $y_t^{\mathrm{st}}$ and exchanging the two finite sums,
\[
y_t^{\mathrm{st}}
=
\sum_{\tau=0}^{t-1}\sum_{a=1}^{r}\lambda_a^{\tau}\,c_a\,u_{t-\tau}
=
\sum_{a=1}^{r} c_a\sum_{\tau=0}^{t-1}\lambda_a^{\tau}\,u_{t-\tau}
=
\sum_{a=1}^{r} c_a\,G_{\lambda_a}(t).
\]
That is, the spectral component of the output is a weighted sum of $r$
single-mode convolutions, one per eigenvalue of $\Lambda$ --- precisely the
objects to which the shift identity \eqref{eq:residual-mode} applies.
Every $\lambda_a$ in the spectral group satisfies
$|\lambda_a|>1-1/\taustb$ (in particular $\lambda_a\neq 0$), so
\eqref{eq:residual-mode} applies to each mode. Summing
\eqref{eq:residual-mode} against $c_a$ over $a=1,\dots,r$,
\begin{equation}
\label{eq:stable-residual}
y_t^{\mathrm{st}}-\sum_{i=1}^{\ustb}\alpha_i y_{t-i}^{\mathrm{st}}
=
\sum_{a=1}^{r}\widetilde c_a\,G_{\lambda_a}(t)
+
\sum_{i=0}^{\ustb-1} D_i\,u_{t-i},
\qquad
\widetilde c_a := \frac{p_\alpha(\lambda_a)}{\lambda_a^{\ustb}}\,c_a,
\quad
D_i := \sum_{a=1}^{r} c_a\,r_i(\lambda_a).
\end{equation}
The first sum on the right-hand side of \eqref{eq:stable-residual} is, at the
same time $t$, the output of the LDS
$(s_t,B_{\mathrm{st}},\widetilde C_{\mathrm{st}},\Lambda)$, where
\[
\widetilde C_{\mathrm{st}}
:=
C_{\mathrm{st}}\operatorname{diag}\!\Bigl(\frac{p_\alpha(\lambda_a)}{\lambda_a^{\ustb}}\Bigr)_{a=1}^{r}:
\]
it has the same eigenvalues $\lambda_1,\dots,\lambda_r$ as the original
spectral block, and by \eqref{eq:rescale-bound} its readout satisfies
$\|\widetilde C_{\mathrm{st}}\|_F\le\widetilde L_{\mathrm{ustb}}\|C_{\mathrm{st}}\|_F$.
By the choice of $h$ and \cref{lem:sf-diagonal-real} applied to
$(s_t,B_{\mathrm{st}},\widetilde C_{\mathrm{st}},\Lambda)$ with target error
$\eta$, there are matrices $W_{s}\in\mathbb R^{\outputDim\times\inputDim}$
such that
\begin{equation}
\label{eq:sf-error}
\sum_{t=1}^T
\Bigl\|
\sum_{a=1}^{r}\widetilde c_a\,G_{\lambda_a}(t)
-
\sum_{s=1}^h W_{s}(\tilde u_t\phi_s)
\Bigr\|_2^2
\le \eta,
\end{equation}
and, for every $s$, $\|W_{s}\|_F\le \widetilde L_{\mathrm{ustb}}R_s$ (the bound
of \cref{lem:sf-diagonal-real} with $\|C\|_F$ replaced by
$\|\widetilde C_{\mathrm{st}}\|_F\le\widetilde L_{\mathrm{ustb}}\|C_{\mathrm{st}}\|_F$).
The constant in $h$ is exactly \cref{lem:sf-diagonal-real}'s requirement for
this rescaled system: the factor $\widetilde L_{\mathrm{ustb}}^2$ absorbs the
rescaled readout, the factor $(1+\ustb L_{\mathrm{ustb}})^2$ absorbs $1/\eta$,
and $K^2$ is the correct power (\cref{lem:sf-diagonal-real}'s $C_0$ carries
$K^2$).
For the boundary matrices, \eqref{eq:rescale-bound} and Cauchy--Schwarz over
$a$ give
\begin{equation}
\label{eq:D-bound}
\|D_i\|_2
\le
\widetilde L_{\mathrm{ustb}}\sum_{a=1}^{r}\|c_a\|_F
\le
\widetilde L_{\mathrm{ustb}}\,\|C_{\mathrm{st}}\|_F\,\|B_{\mathrm{st}}\|_F,
\qquad i=0,\dots,\ustb-1.
\end{equation}
Note that, in contrast with the shifted construction, no approximation of the
lagged outputs $y_{t-i}^{\mathrm{st}}$ is performed: the AR residual of
$y^{\mathrm{st}}$ is approximated directly at time $t$, so no shifted features
$\tilde u_{t-i}\phi_s$ appear and the spectral-filter error \eqref{eq:sf-error}
will not be amplified by the AR coefficients.

\paragraph{Strictly stable component via finite-memory filter.}
The eigenvalues of $A_{\mathrm{fast}}$ have modulus at most $1-1/\taustb$, so
in \cref{lem:fast-fir} the stability parameter is $\taustb$. Applying that
lemma with target error $\eta$ gives a required memory length
\[
L_\eta
=
\left\lceil
\taustb\log\!\left(
\frac{K\,C_{\mathrm{fast},0}\,\taustb\sqrt{T}}{\sqrt{\eta}}
\right)
\right\rceil
=
\left\lceil
\taustb\log\!\left(
\frac{2K\,C_{\mathrm{fast},0}\,\taustb\sqrt{T}(1+\ustb L_{\mathrm{ustb}})}
{\sqrt{\varepsilon}}
\right)
\right\rceil .
\]
The theorem uses the slightly larger, cleaner choice of $L$ in the statement,
which dominates $L_\eta$ up to harmless constants and leaves room for the
later multiplication by $(1+\ustb L_{\mathrm{ustb}})$ in the AR error
propagation for this block.
By \cref{lem:fast-fir} applied to the subsystem
$f_t=A_{\mathrm{fast}}f_{t-1}+B_{\mathrm{fast}}u_t$,
$y_t^{\mathrm{fast}}=C_{\mathrm{fast}}f_t$, with the choice of $L$ in the
theorem statement there are matrices
$\widetilde Q_0,\dots,\widetilde Q_{L-1}\in\mathbb R^{\outputDim\times\inputDim}$
such that, with
\[
\hat y_t^{\mathrm{fast}} := \sum_{\ell=0}^{L-1}\widetilde Q_\ell\,u_{t-\ell},
\qquad
\sum_{t=1}^T\|y_t^{\mathrm{fast}}-\hat y_t^{\mathrm{fast}}\|_2^2 \le \eta,
\]
and each $\|\widetilde Q_\ell\|\le C_{\mathrm{fast},0}$. Set
$\hat y_t^{\mathrm{fast}}:=0$ for $t\le 0$.

It will be convenient to record now the AR residual of this FIR approximation,
which is again a finite-memory filter, of length $L+\ustb$:
\[
\hat y_t^{\mathrm{fast}}-\sum_{i=1}^{\ustb}\alpha_i\hat y_{t-i}^{\mathrm{fast}}
\;=\;
\sum_{\ell=0}^{L+\ustb} F_\ell\,u_{t-\ell},
\qquad
F_\ell
\;:=\;
\widetilde Q_\ell
\;-\;
\sum_{\substack{i=1,\dots,\ustb\\ 0\le \ell-i\le L-1}}\alpha_i\,\widetilde Q_{\ell-i},
\]
for $\ell=0,\dots,L+\ustb$, with the conventions $\widetilde Q_\ell:=0$ for
$\ell\ge L$ and that the sum is zero if no $i$ satisfies the conditions. Using
$\|\widetilde Q_\ell\|_2\le C_{\mathrm{fast},0}$ and
$|\alpha_i|\le L_{\mathrm{ustb}}$, we obtain
\[
\|F_\ell\|_2
\le
(1+\ustb L_{\mathrm{ustb}})\,C_{\mathrm{fast},0}.
\]

\paragraph{Combining the three pieces.}
Substituting
$y_{t-i}^{\mathrm{ustb}} = y_{t-i}-y_{t-i}^{\mathrm{st}}-y_{t-i}^{\mathrm{fast}}$
into the AR identity gives
\[
\begin{aligned}
y_t
&=
\sum_{i=1}^{\ustb}\alpha_i y_{t-i}
+
\sum_{i=0}^{\ustb-1}N_i u_{t-i}
\\
&\quad+
\Bigl(y_t^{\mathrm{st}}-\sum_{i=1}^{\ustb}\alpha_i y_{t-i}^{\mathrm{st}}\Bigr)
+
\Bigl(y_t^{\mathrm{fast}}-\sum_{i=1}^{\ustb}\alpha_i y_{t-i}^{\mathrm{fast}}\Bigr),
\qquad t=1,\dots,T.
\end{aligned}
\]
Define
\[
\pred_t
:=
\sum_{i=1}^{\ustb}\alpha_i y_{t-i}
+
\sum_{i=0}^{\ustb-1}N_i u_{t-i}
+
\Bigl(
\sum_{s=1}^h W_{s}(\tilde u_t\phi_s)
+
\sum_{i=0}^{\ustb-1} D_i\,u_{t-i}
\Bigr)
+
\Bigl(\hat y_t^{\mathrm{fast}}-\sum_{i=1}^{\ustb}\alpha_i\hat y_{t-i}^{\mathrm{fast}}\Bigr),
\]
i.e.\ the AR residual of $y^{\mathrm{st}}$ is replaced by its time-$t$
approximation from \eqref{eq:stable-residual}--\eqref{eq:sf-error}, while the
AR residual of $y^{\mathrm{fast}}$ is replaced by the AR residual of its FIR
approximation. Let
\[
\delta_t
:=
\sum_{a=1}^{r}\widetilde c_a\,G_{\lambda_a}(t)
-
\sum_{s=1}^h W_{s}(\tilde u_t\phi_s),
\qquad
\gamma_t:=y_t^{\mathrm{fast}}-\hat y_t^{\mathrm{fast}}
\]
(with the convention that both are zero for $t\le 0$). By
\eqref{eq:stable-residual},
\[
y_t-\pred_t
=
\delta_t
+
\Bigl(\gamma_t-\sum_{i=1}^{\ustb}\alpha_i\gamma_{t-i}\Bigr),
\qquad t=1,\dots,T.
\]
The first term is bounded directly by \eqref{eq:sf-error}; only the second
term passes through the AR filter. By the triangle inequality (in
$\ell_2(\{1,\dots,T\})$), Young's inequality for convolution, and the bound
$|\alpha_i|\le L_{\mathrm{ustb}}$,
\[
\begin{aligned}
\Bigl(\sum_{t=1}^T\|y_t-\pred_t\|_2^2\Bigr)^{1/2}
&\le
\Bigl(\sum_t\|\delta_t\|_2^2\Bigr)^{1/2}
+
(1+\ustb L_{\mathrm{ustb}})
\Bigl(\sum_t\|\gamma_t\|_2^2\Bigr)^{1/2}
\\
&\le
\sqrt{\eta}
+
(1+\ustb L_{\mathrm{ustb}})\sqrt{\eta}
\le
\sqrt{\varepsilon},
\end{aligned}
\]
so $\sum_{t=1}^T\|y_t-\pred_t\|_2^2\le\varepsilon$.

\paragraph{Putting everything together.}
Collecting all coefficient matrices that multiply lagged inputs, define
\[
Q_\ell
:=
\mathbf 1\{\ell\le \ustb-1\}\,\bigl(N_\ell + D_\ell\bigr)
+
F_\ell,
\qquad
\ell=0,\dots,L+\ustb.
\]
Then
\[
\sum_{i=0}^{\ustb-1}N_i u_{t-i}
+
\sum_{i=0}^{\ustb-1}D_i u_{t-i}
+
\sum_{\ell=0}^{L+\ustb}F_\ell\,u_{t-\ell}
=
\sum_{\ell=0}^{L+\ustb}Q_\ell\,u_{t-\ell},
\]
so $\pred_t$ has exactly the claimed form
\[
\pred_t
=
\sum_{i=1}^{\ustb}\alpha_i y_{t-i}
+
\sum_{s=1}^h W_{s}(\tilde u_t\phi_s)
+
\sum_{\ell=0}^{L+\ustb}Q_\ell\,u_{t-\ell}.
\]
The bound $|\alpha_i|\le L_{\mathrm{ustb}}$ was established in the AR step,
and $\|W_{s}\|_F\le \widetilde L_{\mathrm{ustb}}R_s$ in the spectral step.
Finally, by the triangle inequality, $\|N_\ell\|_2\le \inputDim
L_{\mathrm{ustb}}$ (from $\|N_\ell\|_\infty\le L_{\mathrm{ustb}}$),
\eqref{eq:D-bound}, and the bound on $\|F_\ell\|_2$,
\[
\|Q_\ell\|_2
\le
\inputDim L_{\mathrm{ustb}}
+
\widetilde L_{\mathrm{ustb}}\,\|C_{\mathrm{st}}\|_F\,\|B_{\mathrm{st}}\|_F
+
(1+\ustb L_{\mathrm{ustb}})\,C_{\mathrm{fast},0}
=
L_Q.
\]
This completes the proof.

\paragraph{Memory complexity.}
The memory complexity of the algorithm is $O(\ustb + h + L) = O(\ustb + h + L)$. Since 
$\log{L_{\mathrm{ustb}}} = \tilde{O}(\ustb)$, we get that the memory complexity is $O(\ustb + h + L) = \tilde{O}(\ustb)$.
\end{proof}

\subsection{Robustness to the spectral cutoff}
\label{app:cutoff-robustness}

Here $\ustb$ is a fixed AR budget, which need not equal the instability
complexity of the system under \cref{def:ustb}. We relax the decay assumption
on the finite-memory block while retaining the other block and norm
conditions. As in \cref{lem:ar-plus-sf}, $\varepsilon$ denotes the
target cumulative squared approximation error; $\delta$ denotes the
cutoff shift.

\begin{proposition}[Robustness to the cutoff]\label{prop:cutoff-robustness}
Fix $T\ge2$, $\ustb\ge1$, and $0<\varepsilon\le T$, and write
\[
  \tau:=\log(\ustb+2),\qquad r_0:=1-\frac1\tau.
\]
Consider a noiseless LDS with $x_0=0$ and $\|u_t\|_2\le K$, where $K>0$.
Suppose it admits the real block decomposition of
\cref{lem:ar-plus-sf}, with an AR block of dimension $q\le\ustb$, a
diagonal real spectral block whose eigenvalues satisfy
$r_0<|\lambda_a|<1$, and a semisimple finite-memory block satisfying
\[
  \rho(A_{\mathrm{fast}})\le r_0+\delta,
  \qquad 0\le\delta\le\frac1{2\tau}.
\]
Define $L_{\mathrm{ustb}}$, $\widetilde L_{\mathrm{ustb}}$, $R_s$,
$L_Q$, and $G:=C_{\mathrm{fast},0}$ from these blocks as in
\cref{lem:ar-plus-sf}. Choose $h,L$ by that theorem's formulas with
approximation level $\varepsilon$ and the nominal split parameter $\tau$.

Then the same unified predictor class, with AR order $\ustb$, filter
count $h$, and input lags $0,\ldots,L+\ustb$, contains a predictor
$f^\star$ satisfying
\begin{equation}\label{eq:cutoff-robustness}
  \sum_{t=1}^T\|y_t-f_t^\star\|_2^2
  \le \varepsilon\bigl[1+16KG\delta\log^2(\ustb+2)\bigr].
\end{equation}
Here $\varepsilon$ is the target cumulative squared approximation error,
as in \cref{lem:ar-plus-sf}: at the nominal cutoff $\delta=0$, the
guarantee is
\[
  \sum_{t=1}^T\|y_t-f_t^\star\|_2^2\le \varepsilon.
\]
For fixed $\varepsilon,G,K,\ustb$, the perturbed bound is $\varepsilon+O(\delta)$:
the cutoff shift adds only $O(\delta)$ to the cumulative
approximation error guarantee.
All coefficient norm bounds in \cref{lem:ar-plus-sf} remain valid.
\end{proposition}

\begin{proof}
Use the construction and error decomposition from the proof of
\cref{lem:ar-plus-sf}. Only the finite-memory error estimate changes.
Write
\[
  M:=1+\ustb L_{\mathrm{ustb}}\ge2,\qquad
  \theta:=\delta\tau\in[0,1/2].
\]
With zero-padded inputs, the error from retaining only the first $L$
impulse-response terms of the fast block is
\begin{align*}
  \gamma_t
  &:=y_t^{\mathrm{fast}}-
    \sum_{\ell=0}^{L-1}
    C_{\mathrm{fast}}A_{\mathrm{fast}}^\ell B_{\mathrm{fast}}\,u_{t-\ell}\\
  &=\sum_{\ell=L}^{t-1}
    C_{\mathrm{fast}}A_{\mathrm{fast}}^\ell B_{\mathrm{fast}}\,u_{t-\ell}.
\end{align*}
Thus $\gamma_t$ consists of the omitted contributions from older inputs.
An empty sum is zero, and we set $\gamma_t=0$ for $t\le0$.

Put $R:=r_0+\delta=1-(1-\theta)/\tau<1$. In the diagonalization
$A_{\mathrm{fast}}=V_{\mathrm{fast}}D_{\mathrm{fast}}V_{\mathrm{fast}}^{-1}$,
we have $\|D_{\mathrm{fast}}^\ell\|_2\le R^\ell$. The definition of $G$
therefore gives
\begin{align*}
  \|C_{\mathrm{fast}}A_{\mathrm{fast}}^\ell B_{\mathrm{fast}}\|_2
  &\le\|C_{\mathrm{fast}}V_{\mathrm{fast}}\|_2
       \|D_{\mathrm{fast}}^\ell\|_2
       \|V_{\mathrm{fast}}^{-1}B_{\mathrm{fast}}\|_2\\
  &\le GR^\ell.
\end{align*}
Applying the triangle inequality and $\|u_s\|_2\le K$ to the omitted tail,
\begin{align*}
  \|\gamma_t\|_2
  &\le\sum_{\ell=L}^{t-1}
    \|C_{\mathrm{fast}}A_{\mathrm{fast}}^\ell B_{\mathrm{fast}}\|_2
    \|u_{t-\ell}\|_2\\
  &\le KG\sum_{\ell=L}^{t-1}R^\ell
   \le KG\sum_{\ell=L}^{\infty}R^\ell
   =\frac{KGR^L}{1-R}.
\end{align*}
Since $1-R=(1-\theta)/\tau$ and
$R^L\le e^{-(1-\theta)L/\tau}$, this gives the uniform bound
\[
  \|\gamma_t\|_2\le b_\theta,
  \qquad b_\theta:=\frac{KG\tau}{1-\theta}e^{-(1-\theta)L/\tau}.
\]
The AR coefficients satisfy
$1+\sum_{i=1}^{\ustb}|\alpha_i|\le1+\ustb L_{\mathrm{ustb}}=M$.
Hence, at each time $t$, another application of the triangle inequality gives
\begin{align*}
  \left\|\gamma_t-\sum_{i=1}^{\ustb}\alpha_i\gamma_{t-i}\right\|_2
  &\le\|\gamma_t\|_2+
       \sum_{i=1}^{\ustb}|\alpha_i|\,\|\gamma_{t-i}\|_2\\
  &\le\left(1+\sum_{i=1}^{\ustb}|\alpha_i|\right)b_\theta
   \le M b_\theta.
\end{align*}
Squaring this bound and summing over the $T$ rounds yields
\begin{equation}\label{eq:cutoff-fast-envelope}
\begin{aligned}
  \sum_{t=1}^T\left\|\gamma_t-
       \sum_{i=1}^{\ustb}\alpha_i\gamma_{t-i}\right\|_2^2
  &\le TM^2b_\theta^2\\
  &=\frac{T(KGM\tau)^2}{(1-\theta)^2}
    e^{-2(1-\theta)L/\tau}
   =:\mathcal B(\theta).
\end{aligned}
\end{equation}
Keeping the nominal memory length,
\begin{equation}\label{eq:cutoff-memory}
  L\ge\tau\log\!\left(
    \frac{(1+G)K(1+\tau)TM^2}{\varepsilon}\right),
\end{equation}
gives $\mathcal B(0)\le \varepsilon^2/(M^2T)$ and
$\mathcal B(1/2)\le4KG\varepsilon\tau$.
For fixed $L$, $\mathcal B$ is convex on $[0,1/2]$, so
\begin{equation}\label{eq:cutoff-linear-penalty}
  \mathcal B(\theta)
  \le(1-2\theta)\mathcal B(0)+2\theta\mathcal B(1/2)
  \le\frac{\varepsilon^2}{M^2T}+8KG\varepsilon\delta\tau^2.
\end{equation}

The same spectral tail estimate and choice of $h$ give
$\sum_t\|d_t\|_2^2\le \varepsilon/M^2$ for the spectral residual error $d_t$.
To combine the errors, write
\[
  e_t:=\gamma_t-\sum_{i=1}^{\ustb}\alpha_i\gamma_{t-i}.
\]
The original error decomposition is $y_t-f_t^\star=d_t+e_t$, and
\eqref{eq:cutoff-fast-envelope} gives
$\sum_{t=1}^T\|e_t\|_2^2\le\mathcal B(\theta)$.
Applying $\|a+b\|_2^2\le2\|a\|_2^2+2\|b\|_2^2$ with
$a=d_t$ and $b=e_t$ at each time, then summing, gives
\begin{align*}
  \sum_{t=1}^T\|y_t-f_t^\star\|_2^2
  &=\sum_{t=1}^T\|d_t+e_t\|_2^2\\
  &\le2\sum_{t=1}^T\|d_t\|_2^2
       +2\sum_{t=1}^T\|e_t\|_2^2\\
  &\le\frac{2\varepsilon}{M^2}+2\mathcal B(\theta)\\
  &\le\frac{2\varepsilon}{M^2}+\frac{2\varepsilon^2}{M^2T}
       +16KG\varepsilon\delta\tau^2\\
  &\le \varepsilon+16KG\varepsilon\delta\tau^2,
\end{align*}
where the penultimate step substitutes \eqref{eq:cutoff-linear-penalty}.
For the last step, $\varepsilon\le T$ and $M\ge2$ imply
\[
  \frac{2\varepsilon}{M^2}+\frac{2\varepsilon^2}{M^2T}
  \le\frac{4\varepsilon}{M^2}\le \varepsilon.
\]

Finally, $R<1$ preserves the bound $G$ on each FIR impulse-response
coefficient, so the coefficient norm argument in
\cref{lem:ar-plus-sf} applies unchanged.
\end{proof}

\section{Formal version of \texorpdfstring{\cref{thm:main-regret}}{Theorem \ref*{thm:main-regret}} and its proof}
\label{app:main-regret-proof}

We first restate \cref{thm:main-regret} with the precise filter count and
the constants made explicit, then prove it.

\begin{theorem}[Formal version of \cref{thm:main-regret}]\label{thm:main-regret-formal}
Fix $\hiddenDim\ge 1$. Consider any observable LDS as in \cref{eq:lds} with
state dimension $\hiddenDim$ and instability complexity $\ustb$. Let
$(u_t)_{t\ge 1}$ be any input sequence whose generated outputs
$(y_t)_{t\ge 1}$ satisfy \cref{ass:controlled-outputs}. Write $k:=\ustb$.
Choose the approximation level $\varepsilon_\star$ as in the proof, and set the
filter count $h$ and the finite-memory length $L$ to the corresponding values
from \cref{lem:ar-plus-sf}. Initialize the VAW instances with the
coefficient-norm bound $R_x$ defined in \cref{lem:vaw-regret}.
Logarithms in the asymptotic rates are truncated below
at $1$ to include small $T$ and $k$. Treating the per-mode LDS parameters as constants,
the filter count satisfies
\[
  h \;=\; O\!\left(k\cdot\log T\cdot\log(\mathrm{poly}(T,k,1/\varepsilon))\right)
  \;=\; \widetilde O(\ustb).
\]
Because the autoregressive coefficients can be exponentially large
($\log L_{\mathrm{ustb}}=O(\ustb\log\ustb)$), the finite-memory length satisfies
\[
  L \;=\; O\!\left(\log(k+2)\log\frac{\mathrm{poly}(T,k)\,L_{\mathrm{ustb}}^2}{\varepsilon}\right)
  \;=\; O\!\left(k\cdot\mathrm{polylog}(T,k,1/\varepsilon)\right)
  \;=\; \widetilde O(\ustb).
\]
Then \cref{alg:unified} achieves
\[
  \frac{1}{T}\sum_{t=1}^{T}\ell_2^2(\pred_t,y_t)
  \le
  O\!\left(
    \frac{k^2\log^3 k\,\log^3 T}{T}
  \right)
\]
using
$O\!\left(\outputDim\bigl(\ustb+\inputDim(L+\ustb+1)+h\inputDim\bigr)\right)$
or $O\!\left(k \cdot \log T\cdot\mathrm{polylog}(T,k)\right) = \widetilde{O}(k)$
learnable parameters in total if we are treating the parameters of the LDS as constants except $T, \ustb$.
\end{theorem}

\begin{proof}[Proof of \cref{thm:main-regret-formal}]
Let $\varepsilon>0$ be an approximation level to be chosen below, and let
$f^\star$ be the comparator from \cref{lem:ar-plus-sf}, instantiated with the
corresponding values of $h$ and $L$. By construction, $f^\star$ belongs to the
linear comparator class used by \cref{alg:unified}.
Decompose the cumulative squared loss as
\begin{align*}
  \sum_{t=1}^T \ell_2^2(\pred_t,y_t)
  &=
  \underbrace{
    \left(
      \sum_{t=1}^T \|\pred_t-y_t\|_2^2
      -
      \sum_{t=1}^T \|f^\star_t-y_t\|_2^2
    \right)
  }_{(I)}
  +
  \underbrace{
    \sum_{t=1}^T \|f^\star_t-y_t\|_2^2
  }_{(II)} .
\end{align*}

Term $(I)$ is exactly the regret term controlled by \cref{lem:vaw-regret}.
Treating the LDS parameters as constants we have,
\[
  (I)
  \le
  O\left(
  (\ustb+L+h)
  \left[
    \log T+\log(\ustb+L+Th)+\log(\ustb+L+h)+\ustb\log\ustb
  \right]
  \right).
\]

Term $(II)$ is the approximation error of the unified comparator. By
\cref{lem:ar-plus-sf},
\[
  (II) = \sum_{t=1}^T \|f^\star_t-y_t\|_2^2\le \varepsilon.
\]
Combining the two bounds gives
\[
  \sum_{t=1}^T \ell_2^2(\pred_t,y_t)
  \le
  O\left(
  (\ustb+L+h)
  \left[
    \log T+\log(\ustb+L+Th)+\log(\ustb+L+h)+\ustb\log\ustb
  \right]
  \right)
  +
  \varepsilon.
\]

We now substitute the asymptotic dependence of $h$ and $L$ on
$T,\ustb,\varepsilon$. From \cref{lem:ar-plus-sf} and
$\log L_{\mathrm{ustb}}=O(\ustb\log\ustb)$,
\[
  h
  =
  O\!\left(
    \log T\left[
      \log\frac{T}{\varepsilon}
      +
      \ustb\log(\ustb+2)
    \right]
  \right),
\]
and
\[
  L
  =
  O\!\left(
    \log(\ustb+2)\left[
      \log\frac{T}{\varepsilon}
      +
      \ustb\log(\ustb+2)
    \right]
  \right).
\]
Set
\[
  A_T := \log T+\ustb\log(\ustb+2),
  \qquad
  B_T := \log T+\log(\ustb+2),
  \qquad
  S_\varepsilon := A_T+\log\frac{1}{\varepsilon}.
\]
Then $h+L=O(B_T S_\varepsilon)$, and the logarithmic factor in the regret
bound is $O(S_\varepsilon)$. Therefore
\[
  \sum_{t=1}^T \ell_2^2(\pred_t,y_t)
  \le
  O\!\left(
    \left(\ustb+B_T S_\varepsilon\right)S_\varepsilon
  \right)
  +
  \varepsilon.
\]
The right-hand side is minimized, up to constants hidden in the big-$O$
notation, when the marginal decrease from reducing the logarithmic regret term
balances the marginal increase in approximation error:
\[
  \varepsilon_\star
  =
  \Theta\!\left(\ustb+B_TA_T\right).
\]
With this choice, $S_{\varepsilon_\star}=O(A_T)$ up to lower-order logarithmic
terms, and hence
\[
  \frac{1}{T}\sum_{t=1}^T \ell_2^2(\pred_t,y_t)
  \le
  \frac{
  O\!\left(
    \left(\ustb+B_TA_T\right)A_T
  \right)
  }{T}.
\]
Equivalently, writing $k=\ustb$ and using the crude bounds
$\log T+\log k=O(\log T\log k)$ and
$\log T+k\log k=O(k\log T\log k)$, this is
\[
  \frac{1}{T}\sum_{t=1}^T \ell_2^2(\pred_t,y_t)
  \le
  O\!\left(
    \frac{k^2\log^3 k\,\log^3 T}{T}
  \right).
\]

Finally, the number of learnable parameters is the dimension of the unified
feature vector times the number of output coordinates:
\[
  \outputDim\cdot\hiddenDim_{\mathrm{feat}}
  =
  \outputDim\bigl(\ustb+\inputDim(L+\ustb+1)+h\inputDim\bigr)
  =
  O\!\left(\outputDim\bigl(\ustb+\inputDim L+\inputDim h\bigr)\right).
\]
Substituting
$h=O\!\left(\ustb \cdot \log T\cdot\log(\mathrm{poly}(T,\ustb))\right)=\widetilde O(\ustb)$
and, using $\log L_{\mathrm{ustb}}=O(\ustb\log\ustb)$,
$L=O\!\left(\ustb\cdot\mathrm{polylog}(T,\ustb)\right)=\widetilde O(\ustb)$,
the feature dimension is $\widetilde O(\ustb)$ and gives the stated
$\widetilde O(\ustb)$ parameter count.
\end{proof}

\section{The VAW forecaster and its per-coordinate regret}
\label{app:vaw}

This appendix collects the standard online-linear-regression machinery used
by \cref{alg:unified}. The Vovk--Azoury--Warmuth (VAW) forecaster
\citep{vovk2001competitive} is an online
algorithm for least-squares prediction with logarithmic regret against any
fixed linear comparator. At each time step it maintains a covariance-like
matrix $A_t$ and a target vector $v_t$, and predicts via a forward-step
ridge regression solve.

\begin{algorithm}[H]
  \caption{Vovk--Azoury--Warmuth forecaster}
  \label{alg:vaw}
  \begin{algorithmic}[1]
    \State \textbf{Input:} feature dimension $d$, coefficient-norm bound $S\ge1$.
    \State Initialize $A_0\gets S^{-2}I_d$, $v_0\gets 0\in\mathbb{R}^d$, $t\gets 0$.
    \Function{predict}{$a\in\mathbb{R}^d$}
      \State $t\gets t+1$; store $a_t\gets a$; $A_t\gets A_{t-1}+a_t a_t^\top$.
      \State \Return $\widehat b_t\gets a_t^\top A_t^{-1}v_{t-1}$.
    \EndFunction
    \Function{update}{$b_t\in\mathbb{R}$}
      \State $v_t\gets v_{t-1}+b_t a_t$.
    \EndFunction
  \end{algorithmic}
\end{algorithm}

\begingroup
The input $S$ is an available upper bound on the Euclidean norm of the
comparator coefficients. Its use in the initialization is part of the VAW
subroutine; the prediction algorithm needs no separate regularization
parameter.

\begin{theorem}[{Standard VAW regret bound \citep[Theorem~A.2]{hazan2022introduction}}]\label{thm:vaw-regret}
  Run \cref{alg:vaw} with a coefficient-norm bound $S\ge1$.
  Suppose $\|a_t\|_2\le R$ and $|b_t|\le Y$ for every $t$.
  Then, for every $x^\star\in\mathbb R^d$ with $\|x^\star\|_2\le S$,
  \begin{align*}
    \sum_{t=1}^T(\widehat b_t-b_t)^2
    -\sum_{t=1}^T(a_t^\top x^\star-b_t)^2
    &\le 1+Y^2\log\det\!\left(
      I_d+S^2\sum_{t=1}^T a_ta_t^\top\right)\\
    &\le 1+dY^2\log\!\left(1+\frac{TR^2S^2}{d}\right).
  \end{align*}
\end{theorem}

\begin{proof}
The standard VAW inequality with initial matrix $S^{-2}I_d$ has comparator
penalty $S^{-2}\|x^\star\|_2^2\le1$, which gives the first bound.
For the second, the matrix $\sum_t a_ta_t^\top$ has trace at most $TR^2$.
Concavity of $\log(1+x)$ applied to its $d$ eigenvalues bounds the
log-determinant by $d\log(1+TR^2S^2/d)$.
\end{proof}

We apply this guarantee to the comparator supplied by
\cref{lem:ar-plus-sf}. The bound on its norm is used only to initialize
VAW; the comparator coefficients themselves are unknown.

\begin{lemma}[VAW per-coordinate regret]\label{lem:vaw-regret}
Let $f^\star$ be the comparator from \cref{lem:ar-plus-sf}, and set
\[
  D:=\ustb+\inputDim(L+\ustb+1)+\inputDim h.
\]
Define the feature and coefficient-norm bounds
\begin{equation}\label{eq:vaw-norm-bounds}
  \begin{aligned}
    R_a^2&:=\ustb B_y^2+K^2(L+\ustb+1+T),\\
    R_x^2&:=1+\ustb L_{\mathrm{ustb}}^2+(L+\ustb+1)L_Q^2\\
      &\quad+4\pi Th\,\widetilde L_{\mathrm{ustb}}^2
        \|C_{\mathrm{st}}\|_F^2\|B_{\mathrm{st}}\|_F^2.
  \end{aligned}
\end{equation}
Initialize each VAW instance in \cref{alg:unified} with dimension $D$
and coefficient-norm bound $S=R_x$.
Then the total regret against $f^\star$ satisfies
\[
  R_T:=
  \sum_{t=1}^T\|\pred_t-y_t\|_2^2
  -\sum_{t=1}^T\|f_t^\star-y_t\|_2^2
  \le
  \outputDim+\outputDim B_y^2D
    \log\!\left(1+\frac{T R_a^2R_x^2}{D}\right).
\]
With the remaining norm bounds and input/output dimensions fixed, and
$\log L_{\mathrm{ustb}},\allowbreak
\log\widetilde L_{\mathrm{ustb}},\allowbreak\log L_Q
=O(\ustb\log(\ustb+2))$, this is
\[
  R_T
  =O\!\left(
    (\ustb+L+h)\left[
      \log(T+2)+\log(\ustb+L+h+2)+\ustb\log(\ustb+2)
    \right]\right).
\]
\end{lemma}

Any available upper bounds on the system quantities in
\eqref{eq:vaw-norm-bounds} may be used to form $R_x$, just as for the
choices of $h$ and $L$ in \cref{lem:ar-plus-sf}.

\begin{proof}
For each output coordinate $i$, the label satisfies $|y_t^{(i)}|\le B_y$.
The squared norms of the AR and finite-input feature blocks are at most
$\ustb B_y^2$ and $(L+\ustb+1)K^2$, respectively. For the spectral block,
orthonormality gives
\[
  \sum_{s=1}^h\|\tilde u_t\phi_s\|_2^2
  \le\|\tilde u_t\|_F^2\le TK^2.
\]
Thus $\|a_{t,i}\|_2\le R_a$.

Let $x_i^\star$ concatenate the AR coefficients and the $i$-th rows of the
input and spectral coefficient matrices representing $f_t^{\star(i)}$.
The bounds in \cref{lem:ar-plus-sf} imply
\[
  |\alpha_j|\le L_{\mathrm{ustb}},\qquad
  \|(Q_\ell)^{(i)}\|_2\le L_Q,\qquad
  \|(W_s)^{(i)}\|_2
  \le\|W_s\|_F
  \le 2\sqrt{\pi T}\,\widetilde L_{\mathrm{ustb}}
       \|C_{\mathrm{st}}\|_F\|B_{\mathrm{st}}\|_F.
\]
The last inequality uses the finite-horizon alternative in the definition
of $R_s$ in that theorem. Summing the squared norms of these blocks gives
$\|x_i^\star\|_2^2\le R_x^2-1\le R_x^2$.

Apply \cref{thm:vaw-regret} with $S=R_x$, $R=R_a$, $Y=B_y$, and
dimension $D$ to each coordinate. Each copy has regret at most
$1+B_y^2D\log(1+TR_a^2R_x^2/D)$; summing over the $\outputDim$
coordinates proves the explicit bound.

Finally, $D=O(\ustb+L+h)$ and the definitions give
\[
  \log(1+R_a)+\log R_x
  =O\!\left(\log(T+2)+\log(\ustb+L+h+2)+\ustb\log(\ustb+2)\right).
\]
Substituting into the explicit regret bound proves the stated rate.
\end{proof}
\endgroup
\begingroup
\section{Unknown horizon and instability complexity}
\label{app:unknown-parameters}

The extension is a restart wrapper around \cref{alg:unified}. Run it with
guesses $q$ for the instability complexity and $H$ for the horizon. Double
$q$ if the current run exceeds its loss budget, and double $H$ when time
reaches $H$. The argument is that a sufficient guess cannot fail, and the
costs of smaller guesses form a geometric sum.

\subsection{The algorithm and its loss budget}

Consider the noiseless LDS of \cref{lem:ar-plus-sf} with unknown $k\ge1$.
Assume known input and output bounds $K,B_y\ge1$.
Let $M_{\mathrm{sys}}\ge1$
be a known bound on $C_{\mathrm{fast},0}$,
$\|C_{\mathrm{st}}\|_F\|B_{\mathrm{st}}\|_F$, and the absolute values of
all entries of $A_{\mathrm{ustb}},B_{\mathrm{ustb}},C_{\mathrm{ustb}}$
in the true-$k$ decomposition of \cref{lem:ar-plus-sf}.
These bounds and the input/output dimensions are held fixed.

For guesses $q,H\ge2$, let $\mathcal A(q,H)$ be \cref{alg:unified}
configured using \cref{lem:ar-plus-sf,lem:vaw-regret} with $T=H$, AR order
$q$, and approximation target $1/4$. Evaluate their formulas using the
known upper bounds on the system quantities, conservatively for every
$k\le q$.\footnote{For an explicit choice, substitute
$C_q=[16(\inputDim+1)(1+qM_{\mathrm{sys}})]^{3q+4}$ for each of
$L_{\mathrm{ustb}},\widetilde L_{\mathrm{ustb}},L_Q$, and $M_{\mathrm{sys}}$ for
$C_{\mathrm{fast},0}$ and $\|C_{\mathrm{st}}\|_F\|B_{\mathrm{st}}\|_F$.
These substitutions dominate the true quantities when $q\ge k$:
\cref{lem:ar-plus-sf} gives
$L_{\mathrm{ustb}}\le(\inputDim+1)(1+qM_{\mathrm{sys}})^{2q+1}$,
$\widetilde L_{\mathrm{ustb}}\le(1+qL_{\mathrm{ustb}})12^q$,
and $L_Q\le(\inputDim+2M_{\mathrm{sys}})\widetilde L_{\mathrm{ustb}}$.
In particular, $\log C_q=O(q\log(q+2))$.}
Cap the prescribed filter count at $H$. Write $D,R_a,R_x$ for the resulting
dimension and norm bounds from \cref{lem:vaw-regret}, and define the
computable loss budget
\begin{equation}\label{eq:unknown-loss-budget}
  \mathcal B(q,H)
  :=1+\outputDim+\outputDim B_y^2D
      \log\!\left(1+\frac{H R_a^2R_x^2}{D}\right).
\end{equation}
This is the approximation allowance $1$ plus the existing VAW regret
bound. Substituting the parameter choices into
\cref{lem:ar-plus-sf,lem:vaw-regret}, as in the proof of
\cref{thm:main-regret-formal}, gives
\begin{equation}\label{eq:unknown-budget-growth}
  \begin{aligned}
    D&=O\!\left(q\log^2(q+2)\log^2(H+2)\right),\\
    \mathcal B(q,H)
      &=O\!\left(q^2\log^3(q+2)\log^3(H+2)\right).
  \end{aligned}
\end{equation}

A restart resets the learner and its loss counter. It retains the original
input/output history, so all features continue to use global time.
Project each reported prediction onto the radius-$B_y$ output ball;
this cannot increase squared loss and bounds each round's loss by
$4B_y^2$. The VAW instances continue to receive the true labels.

\begin{algorithm}[H]
\caption{Restarting the unified predictor}
\label{alg:unknown-parameters}
\begin{algorithmic}[1]
  \State Set $q\gets2$, $H\gets2$, and global time $t\gets0$.
  \While{predictions are required}
    \State Start a fresh copy of $\mathcal A(q,H)$; set $\mathcal L\gets0$.
    \Repeat
      \State $t\gets t+1$; run one round of the current copy,
        clipping its reported prediction as above.
      \State $\mathcal L\gets\mathcal L+\|\widehat y_t-y_t\|_2^2$.
    \Until{$\mathcal L>\mathcal B(q,H)$ or $t=H$}
    \If{$\mathcal L>\mathcal B(q,H)$}
      \State $q\gets2q$.
    \EndIf
    \If{$t=H$}
      \State $H\gets2H$.
    \EndIf
  \EndWhile
\end{algorithmic}
\end{algorithm}

Thus a horizon restart preserves the current complexity guess. If both
conditions occur together, both guesses are doubled before the next run.

\subsection{Why the restarts are inexpensive}

\begin{theorem}[Prediction without knowing the horizon or complexity]
\label{thm:unknown-parameters}
Under the assumptions above, \cref{alg:unknown-parameters} needs neither
$T$ nor $k$. It makes at most $O(\log(k+2))$ loss-triggered restarts and
$O(\log(T+2))$ horizon-triggered restarts by time $T$, and satisfies
\[
  \sum_{t=1}^T\|\widehat y_t-y_t\|_2^2
  =O\!\left(k^2\log^3(k+2)\log^4(T+2)\right).
\]
Its active predictor uses
$O(k\log^2(k+2)\log^2(T+2))=\widetilde O(k)$ learnable coefficients.
\end{theorem}

\begin{proof}
\textbf{A sufficient guess never fails.}
If $q\ge k$, the class used by $\mathcal A(q,H)$ contains the
true-$k$ predictor from \cref{lem:ar-plus-sf}: pad its extra AR, input,
and spectral coefficients with zeros. The conservative parameter choices
dominate its required feature counts and coefficient bounds. When all $H$
filters are used, the spectral expansion is exact. Thus there is a
comparator with total approximation error at most $1$ on $[1,H]$ and
coefficient norms at most $R_x$.
The same comparator has error at most $1$ on any run prefix
$I\subseteq[1,H]$. Since the features retain their original history,
\cref{lem:vaw-regret} applies to a fresh learner on $I$ and gives
\[
  \sum_{t\in I}\|\widehat y_t-y_t\|_2^2\le\mathcal B(q,H).
\]
Consequently, $q$ can double only while $q<k$. Starting at $q=2$, this
gives $q\le2k$ and at most $O(\log(k+2))$ loss-triggered restarts.
An insufficient guess may continue indefinitely if its loss stays within
budget; the algorithm need not identify $k$.

\textbf{Each failed run has a bounded cost.}
Just before a loss test fails, the cumulative loss is at most
$\mathcal B(q,H)$. The triggering round adds at most $4B_y^2$, so that
run costs at most $\mathcal B(q,H)+4B_y^2$. Every run that ends without
crossing the budget, including the final unfinished run, costs at most
$\mathcal B(q,H)$, even when $q<k$.

\textbf{The restart costs sum geometrically.}
For a fixed $H$, the complexity guesses are distinct powers of two,
bounded by $2k$. Therefore
\[
  \sum_{\substack{q=2,4,8,\ldots\\q\le2k}}q^2=O(k^2),
\]
and \eqref{eq:unknown-budget-growth} bounds the sum of the run budgets at
that horizon by $O(k^2\log^3(k+2)\log^3(H+2))$.
The horizon guesses are also powers of two, with $H\le2T$ on all observed
rounds. There are $O(\log(T+2))$ such horizons, so summing adds one
factor of $\log(T+2)$. The triggering rounds contribute only
$O(B_y^2\log(k+2))$ in total, proving the claimed loss bound.
Finally, substitute $q\le2k$ and $H\le2T$ into the dimension bound in
\eqref{eq:unknown-budget-growth} to obtain the coefficient count.
\end{proof}

The coefficient count refers to the active predictor; the stored
observation history, spectral filters, and VAW covariance matrices are
additional storage, as in the implementation described above.

\endgroup

\section{Proof of \texorpdfstring{\cref{prop:filter-lb}}{Proposition \ref*{prop:filter-lb}}}
\label{app:filter-lb-proof}

\begin{proof}[Proof of \cref{prop:filter-lb}]
Fix $\ustb\ge 1$, $R<\infty$, and filters
\[
  \Psi^y=\{\psi^y_1,\dots,\psi^y_{q_y}\}\subset \ell_1(\mathbb N),
  \qquad
  \Psi^u=\{\psi^u_1,\dots,\psi^u_{q_u}\}\subset \ell_1(\mathbb N_0),
\]
with $q_y+q_u<\ustb$. In particular, $q_u<\ustb$.

For $s=0,\dots,\ustb-1$ and $a=1,\dots,q_u$, define
\[
  U_{s,a}:=\psi^u_{a,s}.
\]
Thus $U\in\mathbb R^{\ustb\times q_u}$ is the matrix whose columns are the first
$\ustb$ coefficients of the input filters.

We first show that one of the $\ustb$ coordinate shifts cannot be generated by the
input filters. For $j\in\{0,\dots,\ustb-1\}$, let
$P_j:\mathbb R^\ustb\to\mathbb R^{j+1}$ denote the projection onto the first
$j+1$ coordinates. We claim that there exists $j\in\{0,\dots,\ustb-1\}$ such that
\[
  P_j e_j
  \notin
  \operatorname{span}\{P_jU_{\cdot,1},\dots,P_jU_{\cdot,q_u}\}.
\]
Suppose not. Then for every $j=0,\dots,\ustb-1$, there exists
$\beta^{(j)}\in\mathbb R^{q_u}$ such that
\[
  P_jU\beta^{(j)}=P_je_j.
\]
Equivalently, the vector $w^{(j)}:=U\beta^{(j)}\in\mathbb R^\ustb$ satisfies
\[
  w^{(j)}_i=0\quad\text{for all }i<j,
  \qquad
  w^{(j)}_j=1.
\]
Therefore the $\ustb\times\ustb$ matrix with columns
$w^{(0)},w^{(1)},\dots,w^{(\ustb-1)}$ is lower triangular with all diagonal
entries equal to $1$. Hence these $\ustb$ columns are linearly independent. But
all of them lie in the column span of $U$, whose dimension is at most
$q_u<\ustb$, a contradiction.

Fix such a $j$, and define
\[
  \delta
  :=
  \inf_{\beta\in\mathbb R^{q_u}}
  \|P_je_j-P_jU\beta\|_2.
\]
By the choice of $j$, $\delta>0$.

We now construct the LDS. Let $e_0,\dots,e_{\ustb-1}$ be the standard basis of
$\mathbb R^\ustb$. Define $A\in\mathbb R^{\ustb\times\ustb}$ by
\[
  Ae_i=e_{i+1}\quad\text{for }i=0,\dots,\ustb-2,
  \qquad
  Ae_{\ustb-1}=0.
\]
Set $B=e_0$ and $C=e_j^\top$. Then
\[
  \|A\|_\infty,\|B\|_\infty,\|C\|_\infty\le 1.
\]

Choose a block length $L>2\ustb$, to be fixed below. Let
\[
  t_b:=bL+1,\qquad b=0,1,2,\dots.
\]
Set $u_{t_b}=1$ for all $b\ge 0$, and set $u_t=0$ for all other $t$. Thus
$|u_t|\le 1$ for all $t$. Since the system forgets an input after $\ustb$ steps
and $L>2\ustb$, different blocks do not interact through the LDS state. Therefore,
in every block,
\[
  y_{t_b+s}
  =
  \begin{cases}
    1, & s=j,\\
    0, & s\ne j,
  \end{cases}
  \qquad s=0,\dots,L-1.
\]
In particular, $|y_t|\le 1$ for all $t$, so the constructed input sequence is
$1$-controlling in the sense of \cref{ass:controlled-outputs}.

We now choose $L$ large enough so that old blocks are invisible to the filters.
Since all filters are in $\ell_1$, and there are only finitely many of them,
we may choose $L>2\ustb$ so large that, for every input filter $a$ and every
output filter $r$,
\[
  \sum_{\tau\ge L-\ustb}|\psi^u_{a,\tau}|
  \le
  \frac{\delta}{4\max\{R,1\}\sqrt{\ustb}},
  \qquad
  \sum_{\tau\ge L-\ustb}|\psi^y_{r,\tau}|
  \le
  \frac{\delta}{4\max\{R,1\}\sqrt{\ustb}}.
\]

Fix any predictor $f\in\GUFcl_{\Psi^y,\Psi^u}^{R}$, and write
\[
  f_t
  =
  \sum_{r=1}^{q_y}\alpha_r(\psi^y_r*y)_t
  +
  \sum_{a=1}^{q_u}\beta_a(\psi^u_a*u)_t,
  \qquad
  \sum_{r=1}^{q_y}|\alpha_r|
  +
  \sum_{a=1}^{q_u}|\beta_a|
  \le R.
\]
Consider the first $j+1$ times of block $b$,
$t_b,t_b+1,\dots,t_b+j$. On these times, the true output vector is
\[
  (y_{t_b},y_{t_b+1},\dots,y_{t_b+j})^\top=P_je_j.
\]
Consider first an input filter $\psi^u_a$ at a time $t_b+s$, where
$0\le s\le j$. The current block contains a single input pulse at time $t_b$,
so its contribution is exactly the $s$-th filter coefficient,
$\psi^u_{a,s}$. Every other nonzero input comes from an earlier block. Since
blocks are separated by $L$ time steps and the current prefix has length at
most $\ustb$, all such earlier inputs appear at lag at least $L-\ustb$. Hence
\[
  (\psi^u_a*u)_{t_b+s}
  =
  \psi^u_{a,s}+\varepsilon^u_{a,b,s},
  \qquad
  |\varepsilon^u_{a,b,s}|
  \le
  \sum_{\tau\ge L-\ustb}|\psi^u_{a,\tau}|.
\]
Now consider an output filter $\psi^y_r$ on the same prefix
$t_b,\dots,t_b+j$. The current block has no nonzero output before time
$t_b+j$, and at time $t_b+j$ the filter is strictly causal, so it cannot yet
use the value $y_{t_b+j}$. Thus the output-filter feature on this prefix is
entirely due to outputs from previous blocks. Those outputs also have lag at
least $L-\ustb$, and therefore
\[
  |(\psi^y_r*y)_{t_b+s}|
  \le
  \sum_{\tau\ge L-\ustb}|\psi^y_{r,\tau}|,
  \qquad s=0,\dots,j.
\]

Collect the predictions on this prefix of the block into the vector
\[
  F_b:=(f_{t_b},f_{t_b+1},\dots,f_{t_b+j})^\top\in\mathbb R^{j+1}.
\]
The discussion above shows that, on this prefix, the predictor is equal to the
linear combination of the first $j+1$ input-filter coefficients, up to the
negligible contribution of old blocks:
\[
  F_b=P_jU\beta+\xi_b,
  \qquad
  \|\xi_b\|_2\le \frac{\delta}{4}.
\]

The target vector on the same prefix is $P_je_j$. By the definition of
$\delta$, every linear combination of the input-filter prefixes is at least
$\delta$ away from this target:
\[
  \|P_je_j-P_jU\beta\|_2\ge \delta.
\]
The perturbation $\xi_b$ can reduce this distance by at most $\delta/4$, so
\[
  \|P_je_j-F_b\|_2
  \ge
  \|P_je_j-P_jU\beta\|_2-\|\xi_b\|_2
  \ge
  \delta-\frac{\delta}{4}
  =
  \frac{3\delta}{4}.
\]
Thus every predictor in $\GUFcl_{\Psi^y,\Psi^u}^{R}$ incurs a constant amount
of squared error on the prefix of every block:
\[
  \sum_{s=0}^{j}|y_{t_b+s}-f_{t_b+s}|^2
  =
  \|P_je_j-F_b\|_2^2
  \ge
  \frac{9\delta^2}{16}.
\]
If the horizon contains $N$ complete blocks, summing the preceding inequality
over these blocks gives
\[
  \sum_{t=1}^{T}|y_t-f_t|^2
  \ge
  N\cdot \frac{9\delta^2}{16}.
\]
Since $N/T\to 1/L$ as $T\to\infty$, we obtain
\[
  \liminf_{T\to\infty}
  \frac{1}{T}
  \sum_{t=1}^{T}|y_t-f_t|^2
  \ge
  \frac{9\delta^2}{16L}
  >0.
\]
The constant on the right depends only on the fixed filters, $R$, and $\ustb$, not
on the predictor $f$. Taking the infimum over the bounded-coefficient filter
class therefore preserves the positive lower bound:
\[
  \liminf_{T\to\infty}
  \inf_{f\in\GUFcl_{\Psi^y,\Psi^u}^{R}}
  \frac{1}{T}
  \sum_{t=1}^{T}\ell_2^2(f_t,y_t)
  >0.
\]
This proves the theorem.
\end{proof}

\section{Proof of \texorpdfstring{\cref{prop:cond-number}}{Proposition \ref*{prop:cond-number}}}
\label{app:cond-number-proof}

\subsection{Proof for general input dimension}
\label{app:multi-input-idr}

\begin{proof}
We first identify an input direction along which canceling the first input
requires large future input energy. We then convert this energy requirement
into a lower bound on the largest input norm relative to the first, and hence
on the input dynamic range. This gives a deterministic bound depending on the
first input direction; a probability estimate then gives the stated result
for a uniformly random direction. Throughout, $A,B$, and $\eta$ are fixed as
in \cref{prop:multi-input-cond-number}.
Write
\[
  \kappa:=1+\eta,\qquad
  r:=\frac{1+\kappa^{-1}}{2},\qquad
  a:=r\kappa=1+\frac{\eta}{2},\qquad
  D:=rA.
\]
Thus $0<r<1$ and every eigenvalue of $D$ has modulus at least $a>1$.
In particular, $\|D^{-j}\|_2\le a^{-j}$, so powers of $D^{-1}$ decay
geometrically. The choice of $r$ will also give a summable discount on the
future inputs.

\paragraph{A Gramian and its determinant.}
Our goal is to construct an $\inputDim\times\inputDim$ positive semidefinite
matrix $M$ with a large eigenvalue. Begin with the $\ustb\times\ustb$ Gramian
\[
  G:=\sum_{j=1}^{\infty}D^{-j}BB^\top(D^{-j})^\top.
\]
Each matrix $D^{-j}B$ maps an input vector into state coordinates; the
Gramian collects these maps through their products with their transposes.
The series converges in operator norm because
\[
  \|D^{-j}BB^\top(D^{-j})^\top\|_2
  \le a^{-2j}\|B\|_2^2,
\]
and $a>1$. For every $z\in\mathbb R^{\ustb}$,
\[
  z^\top Gz
  =\sum_{j=1}^{\infty}\|B^\top(D^{-j})^\top z\|_2^2
  \ge0.
\]
To show that $G$ is invertible, we prove that equality forces $z=0$.
If $z^\top Gz=0$, every nonnegative summand is zero, so
$z^\top D^{-j}B=0$ for every $j\ge1$. Fix an input coordinate
$\ell\in[\inputDim]$. The first $\ustb$ of these equations read
\[
  \sum_{i=1}^{\ustb}z_i B_{i\ell}(r\lambda_i)^{-j}=0,
  \qquad j=1,\ldots,\ustb.
\]
These are $\ustb$ linear equations in the $\ustb$ unknowns $z_iB_{i\ell}$.
Writing $\xi_i=(r\lambda_i)^{-1}$, their coefficient matrix is
$[\xi_i^j]_{j,i=1}^{\ustb}$, a Vandermonde matrix multiplied by
$\operatorname{diag}(\xi_1,\ldots,\xi_{\ustb})$. The numbers $\xi_i$ are
distinct and nonzero, so this matrix is invertible. Thus $z_iB_{i\ell}=0$
for every $i,\ell$. Every row of $B$ contains a nonzero entry, which forces
every $z_i=0$. This proves $G\succ0$.

We can therefore define
\[
  M:=B^\top G^{-1}B\in\mathbb R^{\inputDim\times\inputDim}.
\]
This matrix is positive semidefinite: for every $v\in\mathbb R^{\inputDim}$,
$v^\top Mv=(Bv)^\top G^{-1}(Bv)\ge0$.

The definition of $G$ gives a useful identity. Multiplying by $D$ and
$D^\top$ shifts each power in the series by one:
\begin{align*}
  DGD^\top
  &=\sum_{j=1}^{\infty}
    D^{-(j-1)}BB^\top(D^{-(j-1)})^\top\\
  &=BB^\top+\sum_{j=1}^{\infty}D^{-j}BB^\top(D^{-j})^\top
   =BB^\top+G.
\end{align*}
Taking determinants of the left-hand side gives
\[
  \det(DGD^\top)=|\det D|^2\det G.
\]
For the right-hand side, the matrix determinant lemma gives
\[
  \det(G+BB^\top)
  =\det G\,\det(I_{\inputDim}+B^\top G^{-1}B)
  =\det G\,\det(I_{\inputDim}+M).
\]
Equating these expressions and canceling $\det G>0$, we obtain
\begin{equation}\label{eq:idr-gramian-determinant}
\begin{aligned}
  \det(I_{\inputDim}+M)
  &=|\det D|^2\\
  &=r^{2\ustb}\prod_{i=1}^{\ustb}|\lambda_i|^2
   \ge a^{2\ustb},
\end{aligned}
\end{equation}
where the inequality uses $|r\lambda_i|\ge a$ for each of the $\ustb$
state coordinates. On the other hand, $M$ has only $\inputDim$ eigenvalues.
Writing them as $\mu_1,\ldots,\mu_{\inputDim}\ge0$ and setting
$\mu=\max_\ell\mu_\ell$, we have
\[
  a^{2\ustb}\le\det(I_{\inputDim}+M)
  =\prod_{\ell=1}^{\inputDim}(1+\mu_\ell)
  \le(1+\mu)^{\inputDim}.
\]
Taking the $\inputDim$-th root yields
\[
  \mu\ge a^{2\ustb/\inputDim}-1
  =\left(1+\frac{\eta}{2}\right)^{2\ustb/\inputDim}-1.
\]
This is where the ratio $\ustb/\inputDim$ enters the proof.
Fix a unit eigenvector $w\in\mathbb R^{\inputDim}$ corresponding to $\mu$.
This direction depends only on $A,B$, and $\eta$, not on the input sequence.

\paragraph{Stabilization requires cancellation.}
Consider any nonzero stabilizing input sequence. If its dynamic range is
infinite, the desired inequality is immediate. Otherwise set
$U:=\sup_t\|u_t\|_2<\infty$. Since $x_0=0$, the state stays zero until the
first nonzero input. We may therefore shift time so that this input is $u_1$.
Unrolling the state recursion displays the contribution of each input:
\[
  x_t=\sum_{s=1}^{t}A^{t-s}Bu_s
     =A^{t-1}Bu_1+A^{t-2}Bu_2+\cdots+Bu_t.
\]
Multiplying by $A^{-(t-1)}$ removes the growth affecting the first input:
\[
  A^{-(t-1)}x_t
  =Bu_1+\sum_{j=1}^{t-1}A^{-j}Bu_{j+1}.
\]
The left side tends to zero, since stabilization gives $x_t\to0$ and
\[
  \|A^{-(t-1)}x_t\|_2\le\kappa^{-(t-1)}\|x_t\|_2\longrightarrow0.
\]
The series on the right converges absolutely because
\[
  \|A^{-j}Bu_{j+1}\|_2\le\kappa^{-j}\|B\|_2U,
  \qquad\kappa>1.
\]
Taking the limit therefore gives
\[
  0=Bu_1+\sum_{j=1}^{\infty}A^{-j}Bu_{j+1}.
\]
Thus the future inputs must cancel the first input after accounting for
their application times through the factors $A^{-j}$. Finally,
$D=rA$ implies $A^{-j}=r^jD^{-j}$, so this necessary cancellation identity
can be written as
\begin{equation}\label{eq:idr-input-cancellation}
\begin{aligned}
  Bu_1=-\sum_{j=1}^{\infty}A^{-j}Bu_{j+1}
      =-\sum_{j=1}^{\infty}D^{-j}B\,(r^j u_{j+1}).
\end{aligned}
\end{equation}
The matrices $D^{-j}B$ are precisely those appearing in $G$, while
$r^ju_{j+1}$ are the discounted future inputs.

\paragraph{From discounted energy to input dynamic range.}
We first lower-bound the future discounted input energy by a quadratic form
in the first input. Set
\[
  Q:=u_1^\top Mu_1\ge0,\qquad z:=G^{-1}Bu_1.
\]
This choice of $z$ ensures that
\[
  z^\top Bu_1=(Bu_1)^\top G^{-1}(Bu_1)=Q.
\]
Taking the inner product of \eqref{eq:idr-input-cancellation} with $z$,
then applying Cauchy--Schwarz to the sum of inner products, gives
\begin{align*}
  Q
  &=\left|\sum_{j\ge1}
      \left\langle B^\top(D^{-j})^\top z,\,r^ju_{j+1}\right\rangle\right|\\
  &\le
  \left(\sum_{j\ge1}\|B^\top(D^{-j})^\top z\|_2^2\right)^{1/2}
  \left(\sum_{j\ge1}r^{2j}\|u_{j+1}\|_2^2\right)^{1/2}.
\end{align*}
By the definition of $G$, the first factor squared is
\[
  \sum_{j\ge1}\|B^\top(D^{-j})^\top z\|_2^2
  =z^\top Gz
  =(Bu_1)^\top G^{-1}(Bu_1)
  =Q.
\]
Hence
\[
  Q\le\sqrt Q\,
    \left(\sum_{j\ge1}r^{2j}\|u_{j+1}\|_2^2\right)^{1/2}.
\]
If $Q>0$, dividing by $\sqrt Q$ and squaring shows that the discounted
energy is at least $Q$; this also holds when $Q=0$. To bound that energy
from above, use $\|u_{j+1}\|_2\le U$ and sum the geometric weights:
\begin{equation}\label{eq:idr-discounted-energy}
\begin{aligned}
  u_1^\top Mu_1
  &\le\sum_{j\ge1}r^{2j}\|u_{j+1}\|_2^2\\
  &\le U^2\sum_{j\ge1}r^{2j}
   =\frac{r^2}{1-r^2}U^2.
\end{aligned}
\end{equation}
Thus a large value of $u_1^\top Mu_1$ forces a large peak input norm:
\[
  U\ge\frac{\sqrt{1-r^2}}{r}\sqrt{u_1^\top Mu_1}.
\]

Write $s=\|u_1\|_2>0$ and $v=u_1/s$, so $\|v\|_2=1$ and
$u_1^\top Mu_1=s^2v^\top Mv$. To use the large eigenvalue $\mu$, decompose
$v=(w^\top v)w+v_\perp$ with $w^\top v_\perp=0$. Since $Mw=\mu w$ and
$M\succeq0$,
\[
  v^\top Mv
  =\mu|w^\top v|^2+v_\perp^\top Mv_\perp
  \ge\mu|w^\top v|^2.
\]
Dividing the peak bound by $s$ and substituting the lower bound on $\mu$
therefore gives
\begin{equation}\label{eq:idr-peak-ratio}
\begin{aligned}
  \frac{U}{s}
  &\ge\frac{\sqrt{1-r^2}}{r}\sqrt{v^\top Mv}\\
  &\ge |w^\top v|\frac{\sqrt{1-r^2}}{r}
       \sqrt{a^{2\ustb/\inputDim}-1}.
\end{aligned}
\end{equation}
The factor $|w^\top v|$ measures how strongly the first input points along
the direction with the large energy requirement.

To isolate the exponential dependence, set $q=\ustb/\inputDim\ge1$ and write
\[
  \frac{\sqrt{1-r^2}}{r}\sqrt{a^{2q}-1}
  =a^q\sqrt{\frac{1-r^2}{r^2}(1-a^{-2q})}.
\]
Since $a>1$ and $q\ge1$, we have
$1-a^{-2q}\ge1-a^{-1}=(1-r)/r$. Consequently
\[
  \frac{1-r^2}{r^2}(1-a^{-2q})
  \ge\frac{1-r^2}{r^2}(1-a^{-1})
  =(1-r)^2\frac{1+r}{r^3}
  \ge(1-r)^2.
\]
Taking square roots and multiplying by $a^q$ gives
\[
  \frac{U}{s}\ge |w^\top v|(1-r)a^{\ustb/\inputDim}.
\]

Finally, the first input is one of the nonzero inputs over which the
infimum in $\IDR(u)$ is taken. Therefore
\[
  \inf\{\|u_t\|_2:u_t\ne0\}\le\|u_1\|_2=s,
  \qquad
  \IDR(u)=\frac{U}{\inf\{\|u_t\|_2:u_t\ne0\}}\ge\frac{U}{s}.
\]
This step does not assume that the first input has the smallest norm:
a smaller later input only increases the dynamic range. Substituting
$1-r=\eta/(2(1+\eta))$ and $a=1+\eta/2$ proves the deterministic
directional inequality
\begin{equation}\label{eq:idr-directional-bound}
  \IDR(u)\ge
  |w^\top v|\,
  \frac{\eta}{2(1+\eta)}
  \left(1+\frac{\eta}{2}\right)^{\ustb/\inputDim}.
\end{equation}
This holds for every stabilizing sequence, using the same $w$. In particular,
when $\inputDim=1$, both $w$ and $v$ are unit scalars, so $|w^\top v|=1$.
This proves the deterministic scalar-input conclusion.

\paragraph{A uniformly random first input direction.}
It remains to show that for fixed unit $w$ and uniform
$v\in\mathbb S^{\inputDim-1}$,
\begin{equation}\label{eq:idr-sphere-probability}
  \Pr\left\{|w^\top v|<\frac{\delta}{\sqrt{\inputDim}}\right\}
  \le\delta,\qquad 0<\delta<1.
\end{equation}
The case $\inputDim=1$ is immediate. For $\inputDim\ge2$, rotational invariance
lets us assume $w=e_1$ and write $v=g/\|g\|_2$, where
$g\sim\mathcal N(0,I_{\inputDim})$. Set $b=\delta/\sqrt{\inputDim}<1$.
Then
\[
  |v_1|<b
  \quad\Longleftrightarrow\quad
  |g_1|<\frac{b}{\sqrt{1-b^2}}\|g_{2:\inputDim}\|_2.
\]
The variable $g_1$ is independent of $g_{2:\inputDim}$, and its density is at
most $1/\sqrt{2\pi}$. Thus $\Pr(|g_1|<c)\le\sqrt{2/\pi}\,c$ for $c\ge0$.
Using also $\mathbb E\|g_{2:\inputDim}\|_2\le\sqrt{\inputDim-1}$ gives
\[
  \Pr(|v_1|<b)
  \le\sqrt{\frac{2}{\pi}}\,
      \frac{b\sqrt{\inputDim-1}}{\sqrt{1-b^2}}
  =\sqrt{\frac{2}{\pi}}\,\delta
      \sqrt{\frac{\inputDim-1}{\inputDim-\delta^2}}
  \le\delta.
\]
Combining \eqref{eq:idr-sphere-probability} with
\eqref{eq:idr-directional-bound} proves \eqref{eq:idr-multiple-input}.
The probability event depends only on the first input direction and the fixed
system; the first input magnitude and all subsequent controls are unrestricted.
Hence the conclusion holds simultaneously over every stabilizing continuation
beginning in that direction, as claimed.
\end{proof}

\paragraph{Why the first-direction qualification is necessary.}
Even with fewer inputs than exploding modes, a specially chosen first
direction can avoid a large dynamic range. For example, take $\inputDim=2$
and $\ustb>2$, choose distinct positive exploding eigenvalues, and let
$B=[b,Ab]$ with every entry of $b$ nonzero. This $B$ has rank two and no zero
rows. The inputs $u_1=(1,0)^\top$, $u_2=(0,-1)^\top$, and $u_t=0$ for $t\ge3$
give $x_1=b$, $x_2=Ab-Ab=0$, and $\IDR(u)=1$. Thus the proposition does not
assert a lower bound uniform over all first input directions.
Moreover, when $\inputDim\ge\ustb$ and $B$ has full row rank, any state can
be canceled in one step. Thus the obstruction depends on the number of input
channels as well as on the first input direction.

\subsection{Alternative proof for a single input}

\begin{proof}
  We prove the deterministic conclusion of \cref{prop:cond-number} for
  $\inputDim=1$, writing $B=(b_1,\ldots,b_\ustb)^\top$.
  Write \(x_t^{(i)}\) for the \(i\)-th coordinate of \(x_t\). Since
  \[
    x_t^{(i)}=\lambda_i x_{t-1}^{(i)}+b_i u_t,
  \]
  we have
  \[
    b_i u_t=x_t^{(i)}-\lambda_i x_{t-1}^{(i)}.
  \]
  If \(u\) stabilizes the system, then \(x_t^{(i)}\to 0\) and
  \(x_{t-1}^{(i)}\to 0\). Hence
  \[
    b_i u_t\to 0.
  \]
  Since \(b_i\ne 0\), it follows that \(u_t\to 0\).
  Therefore, if \(u\) has infinitely many nonzero values, then
  \[
    \inf\{|u_t|:u_t\ne 0\}=0,
  \]
  and so \(\IDR(u)=\infty\). The desired lower bound is then immediate.
  
  It remains to consider the case where \(\IDR(u)<\infty\). Since \(u_t\to 0\)
  and the nonzero values of \(u_t\) are bounded away from zero, \(u\) can have
  only finitely many nonzero entries. Let
  \[
    t_0:=\min\{t:u_t\ne 0\},
    \qquad
    T:=\max\{t:u_t\ne 0\}.
  \]
  Define the polynomial
  \[
    q(z):=\sum_{t=t_0}^{T}u_t z^{t-t_0}.
  \]
  Its constant coefficient is nonzero:
  \[
    q(0)=u_{t_0}\ne 0.
  \]
  
  We claim that
  \[
    q(1/\lambda_i)=0
    \qquad\text{for every }i\in[\ustb].
  \]
  Indeed, for every coordinate \(i\),
  \[
    x_T^{(i)}
    =
    \sum_{t=t_0}^{T}\lambda_i^{T-t}b_i u_t.
  \]
  After time \(T\), the input is zero, so
  \[
    x_s^{(i)}=\lambda_i^{s-T}x_T^{(i)}
    \qquad\text{for all }s\ge T.
  \]
  Since \(|\lambda_i|>1\) and \(x_s^{(i)}\to 0\), we must have
  \[
    x_T^{(i)}=0.
  \]
  Therefore
  \[
    \sum_{t=t_0}^{T}\lambda_i^{T-t}b_i u_t=0.
  \]
  Dividing by \(b_i\ne 0\), we obtain
  \[
    \sum_{t=t_0}^{T}\lambda_i^{T-t}u_t=0.
  \]
  Dividing by \(\lambda_i^{T-t_0}\ne 0\), we obtain
  \[
    \sum_{t=t_0}^{T}u_t\lambda_i^{-(t-t_0)}=0,
  \]
  which is exactly
  \[
    q(1/\lambda_i)=0.
  \]
  
  Now set
  \[
    z_i:=\frac{1}{\lambda_i}.
  \]
  Then \(z_1,\dots,z_\ustb\) are distinct roots of \(q\), and
  \[
    |z_i|\le \frac{1}{1+\eta}.
  \]
  Let
  \[
    U:=\sup_{t\ge 1}|u_t|,
    \qquad
    m:=\inf\{|u_t|:u_t\ne 0\}.
  \]
  Thus
  \[
    \IDR(u)=\frac{U}{m}.
  \]
  
  Choose
  \[
    R:=\frac{1+\frac{1}{1+\eta}}{2}
    =
    \frac{2+\eta}{2(1+\eta)}.
  \]
  Then
  \[
    \frac{1}{1+\eta}<R<1.
  \]
  Hence all roots \(z_1,\dots,z_\ustb\) lie inside the disk \(\{|z|<R\}\).
  
  By Jensen's formula applied to \(q\),
  \[
    \prod_{i=1}^{\ustb} \frac{R}{|z_i|}
    \le
    \frac{\max_{|z|=R}|q(z)|}{|q(0)|}.
  \]
  We now bound the numerator and denominator. First,
  \[
    |q(0)|=|u_{t_0}|\ge m.
  \]
  Second, for every \(|z|=R\),
  \[
    |q(z)|
    \le
    \sum_{t=t_0}^{T}|u_t|R^{t-t_0}
    \le
    U\sum_{j=0}^{\infty}R^j
    =
    \frac{U}{1-R}.
  \]
  Therefore
  \[
    \prod_{i=1}^{\ustb} \frac{R}{|z_i|}
    \le
    \frac{U}{m(1-R)}
    =
    \frac{\IDR(u)}{1-R}.
  \]
  Rearranging gives
  \[
    \IDR(u)
    \ge
    (1-R)\prod_{i=1}^{\ustb}\frac{R}{|z_i|}.
  \]
  Since \(|z_i|=1/|\lambda_i|\), this becomes
  \[
    \IDR(u)
    \ge
    (1-R)R^{\ustb}\prod_{i=1}^{\ustb}|\lambda_i|.
  \]
  Using \(|\lambda_i|\ge 1+\eta\), we get
  \[
    \IDR(u)
    \ge
    (1-R)\bigl(R(1+\eta)\bigr)^{\ustb}.
  \]
  By our choice of \(R\),
  \[
    1-R
    =
    \frac{\eta}{2(1+\eta)}
  \]
  and
  \[
    R(1+\eta)
    =
    \frac{2+\eta}{2}
    =
    1+\frac{\eta}{2}.
  \]
  Therefore
  \[
    \IDR(u)
    \ge
    \frac{\eta}{2(1+\eta)}
    \left(1+\frac{\eta}{2}\right)^{\ustb}.
  \]
  This proves the claim.
  \end{proof}

\section{Experimental details}
\label{app:experimental-details}

This appendix specifies the realization $(A,B,C)$ used in
Section~\ref{sec:experiments} and how the predictors are trained on it.

\subsection{System construction}
\label{app:experimental-system}

The system is the direct sum of four scalar-input, scalar-output subsystems, so
$A$ is block diagonal while $B$ and $C$ are the concatenations of the
per-subsystem input vectors and readout rows,
\[
  \begin{aligned}
    A &= \operatorname{diag}\!\left(
      A_{\mathrm{exp}}, A_{\mathrm{slow}}, A_{\mathrm{fast}}, A_{\mathrm{st}}
    \right), \\
    B &= \left(B_{\mathrm{exp}}^\top, B_{\mathrm{slow}}^\top, B_{\mathrm{fast}}^\top, B_{\mathrm{st}}^\top\right)^\top, \\
    C &= \left(C_{\mathrm{exp}}, C_{\mathrm{slow}}, C_{\mathrm{fast}}, C_{\mathrm{st}}\right),
  \end{aligned}
\]
where $(A_{\mathrm{exp}},B_{\mathrm{exp}},C_{\mathrm{exp}})$ is the exploding
scalar mode, $(A_{\mathrm{slow}},B_{\mathrm{slow}},C_{\mathrm{slow}})$ the slow
complex pair, $(A_{\mathrm{fast}},B_{\mathrm{fast}},C_{\mathrm{fast}})$ the
fast-decaying bank, and $(A_{\mathrm{st}},B_{\mathrm{st}},C_{\mathrm{st}})$ the
stable block, described in turn below. Their state matrices have sizes
$A_{\mathrm{exp}}\in\mathbb{R}^{1\times1}$,
$A_{\mathrm{slow}}\in\mathbb{R}^{2\times2}$,
$A_{\mathrm{fast}}\in\mathbb{R}^{200\times200}$ (a stack of $100$ blocks of size
$2\times2$), and $A_{\mathrm{st}}\in\mathbb{R}^{300\times300}$, for a total state
dimension $\hiddenDim=503$.

\paragraph{Exploding scalar.}
The unstable mode is $A_{\mathrm{exp}}=[1.3]$, with input $B_{\mathrm{exp}}=1$
and readout $C_{\mathrm{exp}}=1$. The input is
$u_t=z_t-1.3z_{t-1}$, where the $z_t$ are independent uniform signs and
$z_0=0$. Starting from zero, the scalar state therefore satisfies
$x_t^{\mathrm{exp}}=z_t$ and remains bounded.

\paragraph{Slow complex pair.}
The slowly decaying pair $0.98\,e^{\pm i1.57}$ is realized as the real
$2\times2$ block
\[
  A_{\mathrm{slow}}
  =
  0.98
  \begin{pmatrix}
    \cos(1.57) & -\sin(1.57) \\
    \sin(1.57) & \cos(1.57)
  \end{pmatrix},
\]
with input $B_{\mathrm{slow}}=(1,0)^\top$ and readout $C_{\mathrm{slow}}=(c,0)$,
where $c>0$ is set so that the peak magnitude of its impulse response equals one.

\paragraph{Fast-decaying bank.}
The bank $A_{\mathrm{fast}}$ consists of $100$ real $2\times2$ rotation--scaling
blocks
\[
  A_{\mathrm{fast},j}
  =
  r_j
  \begin{pmatrix}
    \cos(\theta_j) & -\sin(\theta_j) \\
    \sin(\theta_j) & \cos(\theta_j)
  \end{pmatrix},
\]
with radius $r_j$ uniform in $[0,0.25]$ and angle $\theta_j$ uniform in
$[0,2\pi]$. Each block has input $B_{\mathrm{fast},j}=(1,0)^\top$ and a readout
$C_{\mathrm{fast},j}$ that is a Gaussian row with i.i.d.\ $\mathcal{N}(0,1)$
entries scaled by $0.1$.

\paragraph{Stable block.}
The stable block is $A_{\mathrm{st}}=\operatorname{diag}(\lambda_1,\ldots,
\lambda_{300})$ with input $B_{\mathrm{st}}=\mathbf{1}$, the all-ones vector. We
place the eigenvalues on a logarithmic grid clustered near the stability
boundary: the eigenvalues are $\pm(1-\delta_i)$, where the gaps
$\delta_1,\ldots,\delta_{150}$ are geometrically spaced between $10^{-5}$ and
$1$, giving $150$ positive and $150$ negative nodes.\footnote{Spacing the distance-to-the-circle $1-|\lambda|$
geometrically places many real stable nodes close to $\pm1$, exactly where
long-memory stable responses are hardest for a short finite-memory predictor and
must instead be captured by spectral filters.}
We choose the stable readout $C_{\mathrm{st}}$ so that, after the autoregressive
component cancels the hard-to-learn modes (the $\ustb=3$ eigenvalues counted by
the instability complexity), the residual of the remaining stable response aligns
with the spectral filters, and is thus captured by the unified predictor but not
by autoregression alone. Let $\alpha_1,\alpha_2,\alpha_3$ be the autoregressive coefficients whose roots
are the hard modes $1.3$ and $0.98\,e^{\pm i1.57}$, and define the residual of
the stable impulse response $h_t=\sum_i C_{\mathrm{st},i}\lambda_i^t$ under this
order-three recurrence by
\[
  r_t = h_t-\alpha_1 h_{t-1}-\alpha_2 h_{t-2}-\alpha_3 h_{t-3},
  \qquad t\ge 3.
\]
For each seed we form a target $g=\kappa\sum_{j=1}^{6}s_j w_j\phi_j$ from the
first six spectral filters of \cref{def:spectral-filters}, with random signs
$s_j\in\{-1,+1\}$, weights $w_j$ equally spaced between $1$ and $50$, and
$\kappa$ chosen so that $\|g\|_\infty=25$. We then fit $C_{\mathrm{st}}$ by least
squares so that $r_t$ matches $g_t$ for $t=3,\ldots,T$.

\subsection{Training}
\label{app:experimental-training}

Each predictor is trained by the Vovk--Azoury--Warmuth forecaster of
Appendix~\ref{app:vaw} with regularization $\lambda=0.1$.

\end{document}